\documentclass[twoside,11pt]{article}

\usepackage{blindtext}

\usepackage[preprint]{jmlr2e}

\usepackage{amsmath,amssymb} 
\usepackage{mathrsfs}  
\usepackage{bm}
\usepackage{colortbl}
\usepackage{algorithm}
\usepackage[noend]{algorithmic}
\usepackage{scalerel}
\usepackage{esvect}
\usepackage{nicefrac}
\usepackage{mathtools}

\DeclareMathOperator*{\argmin}{argmin}
\newcommand{\bE}{\mathbb{E}}
\newcommand{\bN}{\mathbb{N}}
\newcommand{\bP}{\mathbb{P}}
\newcommand{\bR}{\mathbb{R}}
\newcommand{\cA}{\mathcal{A}}
\newcommand{\cB}{\mathcal{B}}
\newcommand{\cE}{\mathcal{E}}
\newcommand{\cF}{\mathcal{F}}
\newcommand{\cG}{\mathcal{G}}
\newcommand{\cH}{\mathcal{H}}
\newcommand{\cL}{\mathcal{L}}
\newcommand{\cO}{\mathrm{O}}
\newcommand{\cS}{\mathcal{S}}
\newcommand{\cT}{\mathcal{T}}

\newcommand{\pr}{\mathrm{pr}}
\newcommand{\dd}{\mathrm{d}}
\newcommand{\od}{\bar\dd}
\newcommand{\wb}{w_{\beta}}
\newcommand{\owb}{\overline w_{\beta}}
\newcommand{\lb}{\ell_{\beta}}
\newcommand{\ks}{\mathrm{ks}}
\newcommand{\oks}{\overline\ks}
\newcommand{\md}{\mathrm{d}}
\newcommand{\bSAS}{{\bm{SA\bar S}}}
\newcommand{\bsas}{{\bm{sa\sn}}}
\newcommand{\sn}{\bar s}
\newcommand{\Tn}{\cT^{\circ n}}
\newcommand{\er}{\mathrm{PE}}
\newcommand{\eac}{\mathrm{APE}}
\newcommand{\rmin}{r_{\mathrm{min}}}
\newcommand{\rmax}{r_{\mathrm{max}}}
\newcommand{\vmin}{v_{\mathrm{min}}}
\newcommand{\vmax}{v_{\mathrm{max}}}
\newcommand{\xmin}{x_{\mathrm{min}}}
\newcommand{\xmax}{x_{\mathrm{max}}}
\newcommand{\Li}{\mathrm{Li}}
\newcommand{\bxi}{\bm{\xi}}
\newcommand{\TN}{\Xi}
\newcommand{\TNm}{\Xi_m}
\newcommand{\PN}{\Pi}
\newcommand{\tauG}{T}
\newcommand{\timen}{\mathrm{time}(n)}
\newcommand{\neps}{n^*(\varepsilon)}
\newcommand{\timeeps}{\mathrm{time}(\neps)}
\newcommand{\xilin}{\xi^{\mathrm{lin}}}
\newcommand{\tail}{\mathrm{tail}}
\newcommand{\tailwb}{\mathrm{tail}_{\wb}}
\newcommand{\taillb}{\mathrm{tail}_{\lb}}
\newcommand{\mdp}{\mathfrak{mdp}}
\newcommand{\norm}{\mathrm{Norm}}
\newcommand{\cauchy}{\mathrm{Cauchy}}
\newcommand{\sP}{\mathscr{P}}
\newcommand{\sPR}{\sP(\bR)}
\newcommand{\sPaR}{\sP_{\alpha}(\bR)}
\newcommand{\sPRC}{\sP_{\mathrm{fin}}(\bR)}
\newcommand{\sPRCS}{\sP_{\mathrm{fin}}(\bR)^{\cS}}
\newcommand{\sPbR}{\sP_{\beta}(\bR)}
\newcommand{\sPRS}{\sPR^{\cS}}
\newcommand{\sPbRS}{\sPbR^{\cS}}
\newcommand{\sPdR}{\sP_{\dd}(\bR)}
\newcommand{\sPdRS}{\sP_{\dd}(\bR)^{\cS}}

\usepackage{lastpage}
\jmlrheading{23}{2024}{1-\pageref{LastPage}}{1/21; Revised 5/22}{9/22}{21-0000}{Julian Gerstenberg, Ralph Neininger and Denis Spiegel}

\ShortHeadings{Numerical Policy Evaluation in Distributional Dynamic Programming}{Gerstenberg, Neininger and Spiegel}
\firstpageno{1}

\begin{document}

\title{On Policy Evaluation Algorithms in\\ Distributional Reinforcement Learning}

\author{\name Julian Gerstenberg \email gerstenb@math.uni-frankfurt.de\\
       \addr Institute of Mathematics\\
       Goethe University Frankfurt\\
       Frankfurt am Main, Germany 
       \AND
       \name Ralph Neininger \email neiningr@math.uni-frankfurt.de\\
       \addr Institute of Mathematics\\
       Goethe University Frankfurt\\
       Frankfurt am Main, Germany 
       \AND 
       \name Denis Spiegel \email spiegel@math.uni-frankfurt.de\\
       \addr Institute of Mathematics\\
       Goethe University Frankfurt\\
       Frankfurt am Main, Germany}

\editor{My editor}

\maketitle

\begin{abstract}
We introduce a novel class of  algorithms to efficiently approximate the unknown return distributions in policy evaluation problems from distributional reinforcement learning (DRL). The proposed distributional dynamic programming algorithms are suitable for underlying Markov decision processes (MDPs) having an arbitrary probabilistic reward mechanism, including continuous reward distributions with unbounded support being potentially heavy-tailed. 

For a plain instance of our proposed class of algorithms we prove error bounds, both within
Wasserstein and Kolmogorov--Smirnov distances. Furthermore, for return distributions having probability density functions the algorithms yield approximations for these densities; error bounds are given within supremum norm. We introduce the concept of quantile-spline discretizations to come up with algorithms showing promising results in simulation experiments.

While the performance of our algorithms can rigorously be analysed they can be seen as universal black box algorithms applicable to a large class of MDPs. We also derive new properties of probability metrics commonly used in DRL on which our quantitative analysis is based.

\end{abstract}


\section{Introduction}

In Reinforcement Learning (RL) agents take actions in an environment in order to maximize a cumulative random reward, called return, with respect to its expectation. It has long been also of interest to optimize other characteristics of the probability distribution of the return than its expectation. In recent years, a systematic study of the practical benefits and theoretical foundations of optimizing the return distribution has been taken up; for a comprehensive account of this direction, called Distributional Reinforcement Learning (DRL), see the recently published book \cite{bdr2023}. In the present paper we focus on \emph{Distributional Dynamic Programming} (DDP) algorithms for distributional policy evaluation, see Chapter 5 of  \cite{bdr2023}, i.e., on the construction and analysis of algorithms which yield approximations of the return distribution of a Markov decision process (MDP) when a fixed policy is being considered. 

A crucial ingredient of any MDP is its reward mechanism, which specifies the rewards given to an agent for state-action-state transitions. Such a mechanism implies return distributions which are characterised as fixed-points of the Distributional Bellman Operator (DBO), we recall key results in Section~\ref{sec:mdp}.  DDP algorithms have typically been developed and analysed for MDPs for which all reward distributions are \emph{finitely supported}. A probability distribution $\mu$ on $\bR$ is finitely supported if and only if it has a \emph{representation}  $\mu = \sum_{i=1}^m p_i\delta_{x_i}$ with $m\in\bN$, $p_i \ge 0$ probabilities summing to $1$, \emph{support points} $x_i\in\bR$ and $\delta_{x}$ the Dirac measure in $x\in\bR$. Finitely supported distributions are also called \emph{empirical} in DRL literature. The state of the art for DDP algorithms with finitely supported reward distributions is covered by Chapter 5 of \cite{bdr2023}. Roughly, most of these algorithms operate by iterations of the DBO together with a projection step, which maps a distribution to a fixed-size (i.e., fixed $m$) finitely supported distribution, see also \cite{bell17a,mad17,dabny18b,row18}.

The aim of the present paper is to develop and analyse DDP algorithms that are universally applicable for a wide class of MDPs, including cases of MDPs in which (some of the) reward distributions are not finitely supported, including distributions that have unbounded support and/or are continuous and/or are heavy-tailed. Such MDPs are of interest in applications, for example, in pricing and trading in insurance and finance, see \cite{krasheninnikova2019reinforcement}, \cite{kolm2020modern} and the recent survey \cite{hambly23}. 

To be specific, our DDP algorithms are applicable to any MDP for which cumulative distribution functions (CDFs) as well as quantile functions (QFs) of its reward distributions can be evaluated efficiently, e.g., if all reward distributions are contained in well-known distribution families for which methods to evaluate CDF and QF are efficiently implemented in popular software packages, e.g., in $\mathtt{R}$ or the $\mathtt{Python}$-package $\mathtt{SciPy}$.

The DDP algorithms we are proposing in Section~\ref{sec:framework} also operate by iterations of the DBO together with projection steps that map arbitrary distributions to finitely supported ones. To take up an idea of \cite{dene02}, where fixed-points of related recursive distributional equations were approximated, we allow \emph{varying projections}, i.e.~projections which depend on the update step. In particular, as \cite{dene02}, we allow the sizes of the finitely supported representations to increase with each iteration. 
In the context of the approximation of certain perpetuities this idea has also been applied, see \cite{knne08}.
Furthermore, we allow projections to depend on \emph{previously calculated approximations}, which also seems to generalise previous approaches. A general framework for our family of algorithms is presented in Section~\ref{sec:framework}.

From Section~\ref{sec:metrics} on we investigate the complexity and performance of our algorithms. First, in Section~\ref{sec:metrics} bounds are derived in an abstract setting for our general class of DDP algorithms where we also address how to choose our varying projections, in particular how to let the number of support points grow. We quantify the quality of the calculated approximations by use of \emph{analysis metrics}, such as Wasserstein distances. Our setting and analysis does not require the projections to be non-expansive with respect to the analysis metrics, cf.~Chapters 4 and 5 in \cite{bdr2023} and the discussion in our Section \ref{sec:nonexpansion}. Later, in Section~\ref{sec:ks}, which is also of independent interest, we explain how to extend these bounds to the Kolmogorov--Smirnov distance and how to obtain approximations of densities of return distributions when they exist together with uniform error bounds; we also offer a sufficient criteria for the return distributions to have densities at all and discuss the relation to nonuniform random number generation in Remark \ref{rem_perfect_simulation}.

Concrete universal DDP algorithms are defined and analysed in Section~\ref{sec:ddp-algorithms}. They are universal in the sense that they apply to the whole class of MDPs for which the CDFs and QFs of all reward distributions can be efficiently evaluated, regardless of their nature (discrete, continuous, unbounded support, heavy-tailed, etc.). Our projections are constructed such that they only require that CDFs of the reward distributions can be evaluated. We discuss the trade-off between space-time complexity and approximation quality for a plain instance of the algorithms in Section~\ref{sec:plain}. This yields further evidence that one should let the finitely supported representation sizes increase with each iteration. For practical purpose, we suggest to use the universal DDP algorithms presented in Sections~\ref{sec:plain-adaptive} and \ref{sec:quantile}, which are based on bounding quantiles of convolutions, respextively approximating quantiles by linear-spline interpolations. We conclude with a small simulation study in Section~\ref{sec:simulation} to demonstrate that our algorithms outperform simple Monte-Carlo estimation techniques significantly. 

Our algorithms are reasonably applicable for small finite state-action spaces. For larger spaces methods based on Markov chains exploring the state space and temporal difference learning have been developed, see Chapter 6 of \cite{bdr2023} and \cite{mo10} for a popular particle smoothing approach within Markov chain techniques.

\subsection{Notations}

Let $\cB(\bR)$ be the Borel $\sigma$-field on $\bR$ and $\sPR$ be the set of probability measures on $(\bR,\cB(\bR))$. Let $\sPRC\subsetneq\sPR$ be the subset of finitely supported distributions. Every $\mu\in\sPR$ is specified by its cumulative distribution function (CDF) $F_{\mu}:\bR\to[0,1]$ as well by its quantile function (QF) $F^{-1}_{\mu}:(0,1)\to\bR$ defined as
\begin{equation*}
    F_{\mu}(x) = \mu\big((-\infty,x]\big)\;\;\text{and}\;\;F^{-1}_{\mu}(u) = \inf\{x\in\bR\;|\;F_{\mu}(x)\geq u\},
\end{equation*}
related by $u\leq F_{\mu}(x)\Leftrightarrow F^{-1}_{\mu}(u)\leq x$. The set $\sPR$ is closed under convex combinations, which transfers to the CDF: $F_{q\mu+(1-q)\nu} = q F_{\mu} + (1-q) F_{\nu}$ holds for all $q\in[0,1],\mu,\nu\in\sPR$. Convex combinations do not transfer to the QF. A distribution $\mu\in\sPR$ is said to be continuous if $F_{\mu}$ is continuous and is said to possess the \emph{probability density function} (PDF) $f_{\mu}:\bR\to[0,\infty]$ if $F_{\mu}(x) = \int_{-\infty}^xf_{\mu}(y)\md y$ holds for all $x\in\bR$. 

It is convenient to work with random variables: we write $\cL(X) = \bP[X\in\cdot\;]\in\sPR$ for the distribution of a real-valued random variable (RV) $X$ defined on some probability space $(\Omega,\cF,\bP)$. It is $\cL(X)=\mu$ if and only if $F_{\mu}(x) = \bP[X\leq x]$ for all $x\in\bR$ and a random variable $X$ with $\cL(X)=\mu$ is given by $X:=F_{\mu}^{-1}(U)$, where $U$ is a RV continuous uniform on $(0,1)$. If it exists, let $\bE[X]\in[-\infty,\infty]$ be the expectation of $X$. If $\bP[X\geq 0]=1$, then $\bE[X]\in[0,\infty]$ exists. For $\alpha\in(0,\infty)$ let $\sPaR\subsetneq\sPR$ be the set of distributions $\mu=\cL(X)$ with finite $\alpha$-moment $\bE[|X|^{\alpha}]<\infty$. Note that $\sPRC\subsetneq\sPaR$ and that  $\beta<\alpha$ implies $\sPaR\subsetneq\sPbR$. For an event $F\in \cF$ with $\bP[F]>0$ let $\cL(X|F) = \bP[X\in\cdot\;|F]\in\sPR$ be the conditional distribution and $\bE[X|F] = \bE[X1_F]/\bP[F]$ the conditional expectation of $X$ given $F$ has occurred. 

Finally, for sets $A$ and $I$ it will be convenient to write elements of $A^I$ as $[a_i:i\in I] \in A^I$. For real numbers $x, y\in\bR$ we write $x\vee y = \max\{x,y\}$ and $x\wedge y = \min\{x,y\}$. We use the Bachmann--Landau big-$\mathrm{O}$ notation. 

\section{Distributional Bellman operator and dynamic programming}\label{sec:mdp}

A Markov Decision Process (MDP) with a fixed stationary policy is a tupel
$$\mdp:=\left(\cS,\cA,p,\pi,\gamma,r\right)$$
with finite set of states $\cS\owns s,\sn$, finite set of actions $\cA\owns a$, state-action-state transition probabilities $p = [p(\sn|s,a):(s,a,\sn)\in\cS\times\cA\times\cS]$, in which $p(\sn|s,a)\in[0,1]$ is the probability to make a transition from $s$ to $\sn$ when performing action $a$, a fixed stationary policy $\pi = [\pi(a|s):(s,a)\in\cS\times\cA]$ in which $\pi(a|s)\in[0,1]$ is the probability of the agent choosing action $a$ when in state $s$, a fixed discount factor  $\gamma\in(0,1)$ and $r = [r(\;\cdot\;|s,a,\sn):(s,a,\sn)\in\cS\times\cA\times\cS]$ an arbitrary probabilistic reward mechanism, in which 
$$r(\;\cdot\;|s,a,\sn)=\cL(R_{sa\sn})\in\sPR$$ 
is a probability distribution modelling the distribution of the immediate and possibly random $\bR$-valued reward $R_{sa\sn}$ given to the agent for a transition from state $s$ to $\sn$ having used action $a$. Let $F_{sa\sn}$ be the CDF and $F_{sa\sn}^{-1}:(0,1)\to\bR$ be the QF of $r(\;\cdot\;|s,a,\sn)=\cL(R_{sa\sn})$.

The ingredients of $\mdp$ give rise to the following dynamical process: at time $t=0$ the agent is in a uniform random initial state $S^{(0)}\in\cS$. If at time $t\in\bN_0$ the agent is in state $S^{(t)}$, conditionally on $S^{(t)}$  and the past, the agent chooses a random action $A^{(t)}$ distributed as $\pi(\;\cdot\;|S^{(t)})$. The environment reacts with a random successor state $S^{(t+1)}$ distributed as $p(\;\cdot\;|S^{(t)},A^{(t)})$. The agent receives an immediate real-valued random reward $R^{(t)}$ distributed as $r(\;\cdot\;|S^{(t)},A^{(t)},S^{(t+1)})$. The resulting stochastic process $$(S^{(t)},A^{(t)},R^{(t)})_{t\in\bN_0}$$
is called \emph{full MDP}, and the \emph{return random variable} is defined as
\begin{equation*}
    G^* = \sum\nolimits_{t=0}^{\infty}\gamma^t R^{(t)} = R^{(0)} + \gamma R^{(1)} + \gamma^2 R^{(2)} + \dots
\end{equation*}
which exists with probability one as a $\bR$-valued random variable under the assumption  
\begin{itemize}
    \item[(A1)] $\bE[\log(1\vee|R_{sa\sn}|)]<\infty$ for all $s,a,\sn$,
\end{itemize}
we refer to \cite{genesp23} for a precise result. We assume (A1) holds throughout the paper, but consider stronger moment assumptions on reward distributions later, cf.~Remark~\ref{rem:moment}. Of interest in distributional policy evaluation is the distribution of $G^*$ when starting the full MDP in a fixed but arbitrary given state $s\in\cS$. The (state-)return distribution $\eta^*\in\sPR^{\cS}$ is the $\cS$-indexed collection of all these conditional distributions:
$$\eta^* = [\eta^*_s:s\in\cS]~~\text{with}~~\eta^*_s = \cL(G^*|S^{(0)}=s).$$
We call $\eta^*_s$ the $s$-th component of $\eta^*$. The return distribution $\eta^*$ can rarely by described analytically and the task of distributional policy evaluation algorithms is to approximate $\eta^*$. We discuss algorithms that are applicable in practice under the following two assumptions on the ingredients of $\mdp$:
\begin{itemize}
    \item[(A2)] $F_{sa\sn}(x), F^{-1}_{sa\sn}(u)$ can be calculated in constant time for all $s,a,\sn$ and $x\in\bR, u\in(0,1)$,
    \item[(A3)] $p(\sn|s,a), \pi(a|s) \in [0,1]$ are known for all $s,a,\sn$.
\end{itemize}

Algorithms to approximate $\eta^*$ can be derived from the fact that $\eta^*$ is the unique fixed-point of the \emph{Distributional Bellman Operator} (DBO), which is introduced next. 
In the following, let $\eta = [\eta_{s}:s\in\cS] = [\cL(G_{s}):s\in\cS]\in\sPR^{\cS}$ be an arbitrary $\cS$-indexed collection of distributions and assume that the three collections of random variables $(G_{s})_{s}, (R_{sa\sn})_{sa\sn}$ and the full MDP $(S^{(t)},A^{(t)},R^{(t)})_{t}$ are defined on a common probability space and are independent. 

\begin{definition}\label{def:dbo}
    The DBO is the map $\cT:\sPR^{\cS}\to\sPR^{\cS},\;\eta\mapsto\cT\eta = [\cT_s\eta:s\in\cS]$ defined by its component functions $\cT_s:\sPR^{\cS}\to\sPR$ via $\cT_s\eta :=  \cL(R^{(0)} + \gamma G_{S^{(1)}}|S^{(0)}=s)$.
\end{definition}

\begin{theorem}\label{thm:weak}
    Under (A1) it holds that 
    \begin{itemize}
        \item[(a)] $\eta^*$ is the unique fixed-point of $\cT$, that is for every $\eta\in\sPRS$ it is $\eta=\eta^*$ if and only if $\eta$ solves the so-called distributional Bellman equation $\cT\eta=\eta$,
        \item[(b)] for every $\eta\in\sPRS$ it holds $\Tn\eta\to_w \eta^*$, where $\Tn=\cT\circ\cdots\circ\cT$ is the $n$-th iteration of $\cT$ and $\to_w$ denotes (component-wise) weak convergence of distributions.
    \end{itemize}
\end{theorem}
\begin{proof}
    In case $\bE[|R_{sa\sn}|]<\infty$ for all $s,a,\sn$ this result if well-known and follows, for example, from Proposition~4.34 in \cite{bdr2023}. Under the more general assumption (A1) it follows from Theorem~1 in \cite{genesp23}. 
\end{proof}

Theorem~\ref{thm:weak}(b) leads to a theoretical approximation algorithm for $\eta^*$: choose an initial approximation $\eta^{(0)}$ and calculate a sequence $\eta^{(n)}, n\in\bN$ of approximations by the inductive update rule $\eta^{(n)} = \cT\eta^{(n-1)}$. However, this is not practical for applications, as we explain next. First, by conditioning on $\{A^{(0)}=a, S^{(1)}=\sn\}$, the $s$-th component of $\cT\eta$ equals 
\begin{equation}\label{eq:bellmanoperator-mixture}
    \cT_s\eta = \sum\nolimits_{(a,\sn)\in\cA\times\cS}\pi(a|s)p(\sn|s,a)\cL\left(R_{sa\sn} + \gamma G_{\sn}\right).
\end{equation}
The convolutions $\cL\left(R_{sa\sn} + \gamma G_{\sn}\right)$ also equal (measure theoretic) convex combinations:
\begin{equation}\label{eq:bellmanoperator-integral}
    \cL\left(R_{sa\sn} + \gamma G_{\sn}\right) = \int_{\bR}\cL\left(R_{sa\sn} + \gamma z\right)\md \eta_{\sn}(z) = \int_{\bR}\cL\left(y + \gamma G_{\sn}\right)\md r(y|s,a,\sn).
\end{equation}    
Now, assume all components of $\eta$ are finitely supported with $m\in\bN$ particles, say $\eta_{s} = \sum_{i=1}^{m}w_{s i}\delta_{z_{s i}}$. Combining \eqref{eq:bellmanoperator-mixture} and \eqref{eq:bellmanoperator-integral} gives
\begin{equation}\label{eq:Teta}
    \cT_s\eta = \sum\nolimits_{(a,\sn)\in\cA\times\cS}\sum\nolimits_{i=1}^{m}\pi(a|s)p(\sn|s,a)w_{\sn i}\cL\left(R_{sa\sn} + \gamma z_{\sn i}\right).
\end{equation}
Classical DDP algorithms are developed and analysed under the assumption
\begin{itemize}
    \item[] (A2.Fin): $\;r(\;\cdot\;|s,a,\sn)\in\sPRC$ is finitely supported for all $s,a,\sn$,
\end{itemize}
which implies (A2). If (A2.Fin) holds, that is every reward distribution is finitely supported, say of size $N\in\bN$ with $\cL(R_{sa\sn}) = \sum_{l=1}^N q_{sa\sn l}\delta_{r_{sa\sn l}}$, formula \eqref{eq:Teta} yields 
\begin{equation}\label{eq:Teta-cat}
    \cT_s\eta = \sum\nolimits_{(a,\sn)\in\cA\times\cS}\sum\nolimits_{i=1}^{m}\sum\nolimits_{l=1}^N\pi(a|s)p(\sn|s,a)w_{\sn i}q_{sa\sn l}\delta_{r_{sa\sn l} + \gamma z_{\sn i}},
\end{equation}
which is finitely supported of size (at most) $m\cdot|\cA|\cdot|\cS|\cdot N$ and a representation can be calculated with a space-time complexity of order $\cO(m\cdot|\cA|\cdot|\cS|\cdot N)$, see Algorithm 5.1 in \cite{bdr2023}. As a consequence, if $\eta^{(0)}$ has finitely supported components of size $m$, inductively calculating the $n$-th approximation $\eta^{(n)} = \cT\eta^{(n-1)}$ by applying \eqref{eq:Teta-cat} to $\eta = \eta^{(n-1)}$ is possible in theory, but has a space-time complexity of order $\cO(m\cdot |\cA|^n\cdot|\cS|^n\cdot N^n)$, which is prohibitive for applications. 

If only (A2) holds, the situation is even more complicated: \eqref{eq:Teta} shows that $\cT_s\eta$ can be an intricate mixture involving continuous distributions and/or distributions with unbounded support, in particular it is, in general, no longer finitely supported. However, \eqref{eq:Teta} also shows that the CDF of $\cT_s\eta$ at $x\in\bR$ is given by
\begin{equation}\label{eq:Teta-cdf}
    F_{\cT_s\eta}(x) = \sum\nolimits_{(a,\sn)\in\cA\times\cS}\sum\nolimits_{i=1}^{m}\pi(a|s)p(\sn|s,a)w_{\sn i}F_{sa\sn}\left(x - \gamma z_{\sn i}\right),
\end{equation}
which can be calculated explicitly under (A2) with a time-complexity of order $\cO(m\cdot |\cA|\cdot|\cS|)$. However, calculating the CDF of even the second iterate $\cT^{\circ 2}\eta$ is no longer possible exactly. 

The idea of existing DDP algorithms that work under (A2.FIN) is to prevent the above mentioned space-time complexity blow-up by choosing a \emph{projection map} $\Pi:\sPR\to\sPR$ that maps any $\mu\in\sPR$ to a finitely supported distribution $\Pi(\mu)$ having a \emph{fixed-size} $m$, and to use the update rule $\eta^{(n)} = \bar\Pi\cT\eta^{(n-1)}$, where $\bar\Pi$ is the component-wise extension of $\Pi$ to a map $\sPRS\to\sPRS$. Assuming $\Pi(\mu)$ can be calculated for any finitely supported $\mu$ within controllable space-time-complexity, this leads to useful algorithms, which are discussed and analysed in Chapter 5 of \cite{bdr2023}. Many of these algorithms are not directly applicable under (A2), as the used projections can often not easily be evaluated for distributions which are not finitely supported. Our approach is to focus on a special class of (parameterised) projections $\Pi(\cdot,\xi)$ in which $\Pi(\mu,\xi)$ depends on $\mu$ only by a controllable number of evaluations of its CDF. 
 
\begin{remark}[State-Action return distributions]
    The so-called \emph{state-action}-return distribution is defined as the $(\cS\times\cA)$-indexed collection $\zeta^* = [\zeta^*_{sa}:(s,a)\in\cS\times\cA]$ with $\zeta^*_{sa} = \cL(G^*|S^{(0)}=s, A^{(0)}=a)$. For convenience, we only discuss state-return distributions, but everything could be formulated analogously for the state-action case. This restriction is a recurring theme in the field of distributional policy evaluation, see Section 4 of \cite{mori_10a} or Chapter 5 of \cite{bdr2023}.
\end{remark}

\begin{remark}[Moment Assumptions]\label{rem:moment}
    Each of the following (parameterised) moment assumptions on reward distributions implies (A1):\footnote{In the notations (A1.P$\alpha$), (A1.E$\lambda$), (A1.B) the P indicates polynomial tails, the E exponential tails and B bounded support, where $\alpha$ and $\lambda$ specify the degree and rate respectively.} 
    \begin{itemize}
        \item[] (A1.P$\alpha$), $\alpha\in(0,\infty):\;\;\bE[|R_{sa\sn}|^{\alpha}]<\infty$ for all $s,a,\sn$,
        \item[] (A1.E$\lambda$), $\lambda\in(0,\infty):\;\;\bE[\exp(\lambda|R_{sa\sn}|)]<\infty$ for all $s,a,\sn$,
        \item[] (A1.B)$\;:\;\;$ there is $[\rmin,\rmax]\subset\bR$ with $\bP\big[\rmin\leq R_{sa\sn}\leq \rmax\big] = 1$ for all $s,a,\sn$. 
    \end{itemize}
     Condition (A1.B) implies (A1.E$\lambda$) for every $\lambda$, (A1.E$\lambda$) implies (A1.P$\alpha$) for every $\alpha$ and (A1.P$\alpha$) implies (A1). The bounded rewards case (A1.B) is frequently encountered in the (distributional) reinforcement learning literature, as well as (A1.P$\alpha$) with $\alpha=1$. All of these moment assumptions transfer to the return distributions, let $\eta^*_s=\cL(G^*_s)$: (A1.P$\alpha$) implies $\bE[|G^*_s|^{\alpha}]<\infty$ for all $s$, (A1.E$\lambda$) implies $\bE[\exp(\lambda|G^*_s|)]<\infty$ for all $s$ and (A1.B) via bounding interval $[\rmin,\rmax]\subset\bR$ implies $\bP[\vmin\leq G^*_s\leq \vmax]=1$ for all $s$ with bounding interval $[\vmin,\vmax] = [\frac{\rmin}{1-\gamma},\frac{\rmax}{1-\gamma}]$. Also note that (A2.Fin) implies (A1.B), hence $\eta^*_s$ has compact support in this case. In case of unbounded support, the design and analysis of DDP algorithms is more challenging.   
\end{remark}

\section{A general DDP Framework}\label{sec:framework}

A \emph{parameterised projection} is a map 
\begin{equation*}
    \Pi:\sPR\times\Xi\longrightarrow\sPR,\;\;(\mu,\xi)\longmapsto \Pi(\mu,\xi),  
\end{equation*}
where $\Xi\owns \xi$ is a set of parameters. Parameters are chosen by a \emph{parameter algorithm} $A$, that can depend on ingredients of $\mdp$, and maps as  
\begin{equation*}
    A:\sPRS\times\cS\times\bN\longrightarrow\Xi,\;\;(\eta,s,k)\longmapsto A(\eta,s,k).
\end{equation*}
Given an initial approximation $\eta^{(0)}\in\sPRS$, the maps $\Pi$ and $A$ are used to calculate a sequence
\begin{equation}\label{eq:sequence}
    (\eta^{(n)},\bm{\xi}^{(n)})\in\sPRS\times\Xi^{\cS},\;n\in\bN
\end{equation}
inductively over $n\in\bN$ as follows:
\begin{enumerate}
    \item calculate $\bxi^{(n)} = [\xi^{(n)}_s:s\in\cS]$ with $\xi^{(n)}_s = A(\eta^{(n-1)},s,n)$, 
    \item calculate $\eta^{(n)} = [\eta^{(n)}_s:s\in\cS]$ with $\eta^{(n)}_s = \Pi(\cT_s\eta^{(n-1)},\xi^{(n)}_s)$.
\end{enumerate}

All concrete parameterised projections we consider in this paper calculate finitely supported distributions in which the size of $\Pi(\mu,\xi)\in\sPRC$ depends on the parameter $\xi$. Further, as we allow $A$ to depend on the ingredients of $\mdp$ and $\cT$ is a function of $\mdp$, the parameters $\bxi^{(n)}$ calculated in the $n$-th update step may depend on (properties of) $\cT\eta^{(n-1)}$, see Section~\ref{sec:quantile}. Finally, the DDP framework considered in Chapter 5 of \cite{bdr2023} is included in this setting by considering $\Xi$ to be a one-point set, we elaborate on this in Section~\ref{sec:nonexpansion}.

\begin{example}[Quantile Dynamic Programming, QDP]\label{ex:qdp}
    Let $M:\bN\to\bN$ be a function, which is a hyper-parameter of the algorithm. Let $\Xi=\bN$ and
    \begin{align*}
        \Pi&:\sPR\times\bN\to\sPR,&\Pi(\mu,m) &= \frac{1}{m}\sum\nolimits_{i=1}^m \delta_{x_i}\;~\text{with}\;x_i = F^{-1}_{\mu}\left(\frac{2i-1}{2m}\right),\\
        A&:\sPRS\times\cS\times\bN\to\bN,&A(\eta,s,k) &= M(k).
    \end{align*}
    Let $\eta^{(0)}$ have finitely supported components. The iterative procedure to calculate $\eta^{(n)}, n\in\bN$ as in \eqref{eq:sequence} specifies to  
    \begin{equation}\label{eq:qdp-update}
        \eta^{(n)}_s = \frac{1}{M(n)}\sum\nolimits_{i=1}^{M(n)}\delta_{x_i}\;\text{with}\;x_i = F_{\cT_s\eta^{(n-1)}}^{-1}\left(\frac{2i-1}{2M(n)}\right).
    \end{equation}
    Following Chapter of 5 of \cite{bdr2023}, who introduce and analyse this algorithm for $M(k)\equiv m\in\bN$ constant and under (A2.Fin), we call this the \emph{Quantile Dynamic Programming (QDP)} algorithm. To perform \eqref{eq:qdp-update} in practice, quantiles of $\cT_s\eta^{(n-1)}$ need to be accessible, which holds under (A2.Fin) and leads to Algorithm 5.4 in \cite{bdr2023}, which can be easily modified to work for arbitrary size functions $M$. By analysing the trade-off between approximation quality and time complexity, we argue to choose $M(k)$ growing exponentially in $k$. To be precise, choosing $M(k) = \lceil (1/\theta)^k\rceil$ for some $\theta\in[\gamma^c,1)$, cf.~Example~\ref{ex:qdp-time}. 
\end{example}
    
QDP is directly applicable only in case (A2.Fin), because then, for any finitely supported $\eta$, it is $\cT_s\eta$ finitely supported, see \eqref{eq:Teta-cat}, and hence its QF $F_{\cT_s\eta}^{-1}$ can be explicitly evaluated. Under the more general assumption (A2), only the CDF $F_{\cT_s\eta}$ can be evaluated, see \eqref{eq:Teta-cdf}, which is the CDF of an intricate mixture of continuous and/or unbounded distributions. Inverting CDFs is, in general, only possible with numerical approximation methods, which can be costly. Similar problems arise when aiming to design random number generators that are universally applicable for distributions $\mu$ in which only their CDFs can be evaluated, we refer to Section~7 of \cite{hormann2004automatic} and Chapter~II.2 in \cite{devroye1986}. Note that many of these numerical inversion algorithms assume some sort of regularity on the CDF, which, in general, does not hold under (A2).

Concrete DDP algorithms applicable under (A2) are introduced and analysed in Section~\ref{sec:ddp-algorithms}, which starts by introducing a particular parameterised projection, see Figure~\ref{fig:pi} for a visualisation. In Sections \ref{sec:plain}, \ref{sec:plain-adaptive} and \ref{sec:quantile} we discuss concrete parameter algorithms for this projection. The latter leads to an \emph{approximation of the QDP algorithm applicable under (A2)}, where quantiles are approximated by \emph{linear splines}, but this approximation is \emph{on the level of parameters of the projection, not on the level of designing a projection}.

\section{Bounding the approximation error}\label{sec:metrics}

Let $\dd:\sPR\times\sPR\to[0,\infty]$ be an \emph{analysis metric} on $\sPR$, which is allowed to assign infinite distance $\dd(\mu,\nu)=\infty$ to pairs $\mu\neq \nu$.\footnote{$\dd$ satisfies $\dd(\mu,\nu)=0\Leftrightarrow \mu=\nu$ and $\dd(\mu,\nu)=\dd(\nu,\mu)$ and $\dd(\mu,\nu)\leq \dd(\mu,\zeta)+\dd(\zeta,\nu)$ for all $\mu,\nu,\zeta\in\sPR$.} By triangle inequality, $\dd$ is a proper finite metric restricted to $\sPdR = \{\mu\in\sPR:\dd(\mu,\delta_0)<\infty\}\subseteq\sPR$.\footnote{In Chapter~5 of \cite{bdr2023} it is further required that any $\mu\in\sPdR$ has finite first moment.} Any such $\dd$ is extended to a metric $\od$ on $\sPRS$ by $\od(\eta,\eta') = \max_{s\in\cS}\dd(\eta_s,\eta'_s)$, which is a proper finite metric on $\sPdRS\subseteq\sPRS$. With respect to the analysis metric $\dd$, the quality of an approximation $\eta^{(n)}$ for $\eta^*$ is measured by $\od(\eta^{(n)},\eta^*)$. Two popular parameterised families of analysis metrics are the following: 

The $\beta$-\emph{Wasserstein} distances $\wb, \beta\in(0,\infty]$, are defined as 
\begin{equation*}
    \wb(\mu,\nu) = \begin{cases}
        \int_0^1 |F_{\mu}^{-1}(u)-F_{\nu}^{-1}(u)|^{\beta}\md u,&\beta\in(0,1),\\
        \left[\int_0^1 |F_{\mu}^{-1}(u)-F_{\nu}^{-1}(u)|^{\beta}\md u\right]^{1/\beta},&\beta\in [1,\infty),\\
        \sup_{u\in(0,1)}|F_{\mu}^{-1}(u)-F_{\nu}^{-1}(u)|,&\beta = \infty.
    \end{cases}
\end{equation*}
We have $\sP_{\wb}(\bR) = \sP_{\beta}(\bR)$ for $\beta\in(0,\infty)$ and $\sP_{w_{\infty}}(\bR)$ is the set of distributions with bounded support. 
    
The \emph{Birnbaum--Orlicz average} distances $\lb, \beta\in(0,\infty]$, are defined as 
\begin{equation*}
    \lb(\mu,\nu) = \begin{cases}
        \int_{-\infty}^{\infty} |F_{\mu}(x)-F_{\nu}(x)|^{\beta}\md x,&\beta\in(0,1),\\
        \left[\int_{-\infty}^{\infty} |F_{\mu}(x)-F_{\nu}(x)|^{\beta}\md x\right]^{1/\beta},&\beta\in [1,\infty),\\
        \sup_{x\in\bR}|F_{\mu}(x)-F_{\nu}(x)|,&\beta = \infty.
    \end{cases}
\end{equation*}
It is $\ks := \ell_{\infty}$ the \emph{Kolmogorov--Smirnov} distance, $\ell_2$ is also known as \emph{Cram\'er} distance and $\ell_1 = w_1$. A complete characterisation of all spaces $\sP_{\ell_{\beta}}(\bR), \beta\in(0,\infty]$, in terms of moment assumptions is challenging. However, for our purpose, it is enough to observe the following: for $\beta=1$ we have $\sP_{\ell_1}(\bR) = \sP_{w_1}(\bR) = \sP_1(\bR)$, in case $\beta=\infty$ it is $\sP_{\ell_{\infty}}(\bR) = \sP_{\ks}(\bR) = \sP(\bR)$, that is $\ks$ is a metric on the whole of $\sPR$, and for $\beta\in(1,\infty)$ we have
\begin{equation}\label{eq:lb-domain}
    \sP_{\frac{1}{\beta}}(\bR) \;\subsetneq\; \sP_{\lb}(\bR) \;\subsetneq\; \bigcap\nolimits_{0<\varepsilon<\frac{1}{\beta}}\sP_{\frac{1}{\beta}-\varepsilon}(\bR),
\end{equation}
which is proved in Proposition~\ref{prop:lb-space} in Appendix \ref{app:spaces}.

\subsection{The Accumulated Projection Error}

The following theorem is the key tool to study $\od(\eta^{(n)},\eta^*)$, but also shows that not all of the above introduced metrics are equally admissible as analysis metrics. Let $c\in(0,\infty)$ be a fixed constant. 

\begin{theorem}[Theorem~4.25 of \cite{bdr2023}]\label{thm:hom}
    Let $\dd$ be $c$-homogeneous, regular and convex\footnote{See Definition~\ref{def:properties} in Appendix \ref{app:metric}, also refer to Definitions 4.23-25 in \cite{bdr2023}.}. Then $\cT$ is a $\gamma^c$-contraction with respect to $\od$, that is 
    \begin{equation}\label{eq:contraction}
        \od\left(\cT\eta,\cT\eta'\right)\;\leq\;\gamma^c\cdot \od(\eta,\eta')\;\;\text{for all}\;\eta,\eta'\in\sPRS
    \end{equation}
    and, consequently, for all $\eta\in\sPRS$:
    \begin{itemize}
        \item[(i)] $\od\left(\Tn\eta,\eta^*\right)\;\leq\;\gamma^{cn}\cdot \od(\eta,\eta^*)$ for all $n\in\bN_0$,
        \item[(ii)] $\od(\eta,\eta^*)<\infty\;\;\Longrightarrow\;\;\od(\eta,\eta^*)\leq \frac{\od(\eta,\cT\eta)}{1-\gamma^c}$.
    \end{itemize}
\end{theorem}

Of the above introduced metrics, the following have the properties stated in Theorem~\ref{thm:hom}:
\begin{itemize}
    \item[-] $\dd = \wb, \beta\in(0,\infty]$ with $c=\min\{1,\beta\}$,
    \item[-] $\dd = \lb, \beta\in[1,\infty)$ with $c=1/\beta$.
\end{itemize}

In particular, $\ks = \ell_{\infty}$ is \emph{not} covered by Theorem~\ref{thm:hom}.

\begin{remark}
    The inequalities stated in Theorem~\ref{thm:hom} are of interest in case the distances appearing in the upper bounds are finite. We have that 
    $$\eta,\eta',\eta^*\in \sPdR^{\cS},\;\;\cT(\sPdR^{\cS})\subseteq\sPdR^{\cS} \;\;\Longrightarrow\;\; \od(\eta,\eta'), \od(\eta,\eta^*), \od(\eta,\cT\eta)\;<\infty.$$ 
    By triangle inequality, \eqref{eq:contraction} and properties of $\dd$ it can be shown that the moment assumption 
    \begin{itemize}
        \item[] (A1.$\dd$)$\;:\;\;$ $\cL(R_{sa\sn})\in\sPdR\;\text{for all}\;s,a,\sn$
    \end{itemize}
    implies $\cT(\sPdR^{\cS})\subseteq\sPdR^{\cS}$. This alone may, in general, not be sufficient to conclude $\eta^*\in\sPdR^{\cS}$, but for the metrics we consider the following holds: for $\dd=\wb, \beta\in(0,\infty)$, it is $\sP_{\wb}(\bR)=\sPbR$, hence (A1.$\wb$) is equivalent to (A1.P$\beta$) and this assumption implies $\eta^*\in\sPbRS$, see Remark~\ref{rem:moment}. Similarly, $\sP_{w_{\infty}}(\bR)$ is the set of distributions with bounded support, hence (A1.$w_{\infty}$) is equivalent to (A1.B), which implies that all components of $\eta^*$ have bounded support. Finally, we have $\sP_{1/\beta}(\bR)\subset\sP_{\lb}(\bR)$ for $\beta\in[1,\infty)$, see Proposition~\ref{prop:lb-space}, hence (A1.$\lb$) is implied by (A1.P$\frac{1}{\beta}$), which implies $\eta^*\in\sP_{1/\beta}(\bR)^{\cS}\subseteq\sP_{\lb}(\bR)^{\cS}$. 
\end{remark}

The Kolmogorov--Smirnov distance $\ks=\ell_{\infty}$ does not satisfy the properties stated in Theorem~\ref{thm:hom} and thus is not as easily admissible as, for example, Wasserstein distances. However, we demonstrate in Section~\ref{sec:ks} how bounds on approximation errors $\od(\eta^{(n)},\eta^*)$ using $\dd=\wb$ for some $\beta\in(0,\infty)$ lead to useful bounds for $\oks(\eta^{(n)},\eta^*)$. 

For now, let $\dd$ be a $c$-homogeneous, regular and convex metric on $\sPR$. Let $\eta^{(0)}$ be an initial approximation and let $(\eta^{(n)},\bxi^{(n)}), n\in\bN$ be the sequence calculated inductively as in \eqref{eq:sequence}. A possible approach to bound the approximation error $\od(\eta^{(n)},\eta^*)$ is as follows: define the \emph{$k$-th projection error}($\er$) and the \emph{$n$-th accumulated projection error}($\eac$) by
\begin{align*}
    \er(k,\dd) &= \od(\eta^{(k)},\cT\eta^{(k-1)}),\\
    \eac(n,\dd) &= \sum\nolimits_{k=1}^{n}\gamma^{c(n-k)}\cdot\er(k,\dd).
\end{align*}

\begin{proposition}\label{prop:generalbound}
    $\od\left(\eta^{(n)},\eta^*\right)\;\leq\; \eac(n,\dd)\;+\;\gamma^{cn}\od(\eta^{(0)},\eta^*)$.
\end{proposition}
\begin{proof} 
    By induction on $n\in\bN_0$: for $n=0$ it is obvious. Assuming the inequality holds for some $n\in\bN_0$ and noticing the recursive formula $\eac(n+1,\dd) = \er(n+1,\dd) + \gamma^c \eac(n,\dd)$ it follows
    \begin{align*}
        \od\left(\eta^{(n+1)},\eta^*\right) &\leq \od\left(\eta^{(n+1)},\cT\eta^{(n)}\right) + \od\left(\cT\eta^{(n)},\eta^*\right)\\
        &= \er(n+1,\dd) + \od\left(\cT\eta^{(n)},\cT\eta^*\right) \\
        &\leq \er(n+1,\dd) + \gamma^{c}\od(\eta^{(n)},\eta^*)\\
        &\leq \er(n+1,\dd) + \gamma^{c}\left(\eac(n,\dd)\;+\;\gamma^{cn}\od(\eta^{(0)},\eta^*)\right)\\
        &=\eac(n+1,\dd)\;+\;\gamma^{c(n+1)}\od(\eta^{(0)},\eta^*),
    \end{align*}
    where the first inequality follows from triangle inequality, the equality from $\cT\eta^*=\eta^*$, the second inequality  from Theorem~\ref{thm:hom} and the third by induction hypothesis. 
\end{proof}

Bounds on projection errors yield bounds on the accumulated projection error:

\begin{lemma}\label{lemma:ape-bound}
    Let $r\in(0,\infty), \theta\in(0,1), D>0$. Denote with $\Li_s(z) = \sum\nolimits_{k=1}^{\infty}k^{-s}z^k$ the polylogarithm of order $s$ and argument $z$. Then the following implications hold
    \begin{align*}
        \forall k=1,\dots,n:\;\er(k,\dd)&\leq D &&\Longrightarrow && \eac(n,\dd)\leq \frac{D}{1-\gamma^c},\\
        \forall k=1,\dots,n:\;\er(k,\dd)&\leq D\cdot k^{-r} &&\Longrightarrow && \eac(n,\dd)\leq \frac{D\cdot \Li_{-r}(\gamma^c)}{\gamma^c}\cdot n^{-r},\\
        \forall k=1,\dots,n:\;\er(k,\dd)&\leq D\cdot\theta^k &&\Longrightarrow && \eac(n,\dd)\leq \begin{cases}
                \frac{D}{(\gamma^c/\theta)-1}\cdot \gamma^{cn},&\theta<\gamma^c,\\
                D\cdot n\cdot \gamma^{cn},&\theta=\gamma^c,\\
                \frac{D}{1-\gamma^c/\theta}\cdot \theta^n,&\theta>\gamma^c.
        \end{cases}
    \end{align*}
\end{lemma}
\begin{proof}
    The first and third implications follow from bounding finite geometric series by infinite geometric series, resp.~explicit evaluation in case $\theta=\gamma^c$. The second implication follows from $1/(n-k)\leq (k+1)/n$ for every $k=0,\dots,n-1$, bounding finite by infinite series and plugging in the definition of the polylogarithm.
\end{proof}

Proposition~\ref{prop:generalbound} in combination with Lemma~\ref{lemma:ape-bound} show that if projection errors decay towards zero fast enough as $k\to\infty$, then the approximation error $\od(\eta^{(n)},\eta^*)$ also decays to zero as $n\to\infty$ with a comparable rate of decay. This allows to bound the minimal number of iterations required to calculate an approximation $\eta^{(n)}$ that is $\varepsilon$-close to $\eta^*$ with respect to $\dd$. By investigating the trade-off between time complexity and approximation quality, we make a case for constructing DDP algorithms such that $\er(k,\dd)$ decays exponentially with some rate $\theta\in[\gamma^c,1)$, corresponding to the last (two) parts of the previous lemma, see the following Example \ref{ex:qdp-time} and Section \ref{sec:plain}.

\begin{example}[Analysing QDP with respect to $\dd=w_1$]\label{ex:qdp-time}
    Let $\eta^{(n)}, n\in\bN$ be as in Example~\ref{ex:qdp} having used the size function $M:\bN\to\bN$, that is $\eta^{(n)}_s = \Pi(\cT_s\eta^{(n-1)},M(n))$ with $\Pi(\mu,m) = m^{-1}\sum_{i=1}^m\delta_{F_{\mu}^{-1}(\nicefrac{2i-1}{2m})}$. We consider the analysis metric $\dd=w_1=\ell_1$, which is $c=1$-homogeneous. To analyse QDP, we note the following inequalities: 
    \begin{equation*}
        w_1(\Pi(\mu,m),\mu) \leq \int_0^{\frac{1}{2m}}\left(F_{\mu}^{-1}(1-u) - F_{\mu}^{-1}(u)\right)\md u,\;\;\;w_1(\mu,\nu) \leq w_{\infty}(\mu,\nu).
    \end{equation*}
    As discussed, QDP is practical in case (A2.Fin) holds, which implies that reward distributions are supported on some compact interval $[\rmin,\rmax]\subset\bR$. Let $[\vmin,\vmax] = \left[\rmin/(1-\gamma),\rmax/(1-\gamma)\right]$. If all components of $\eta\in\sPRS$ are supported on $[\vmin,\vmax]$, then all components of $\cT\eta$ are supported on $[\vmin,\vmax]$ and the same holds for the components of $\eta^*$. Assume $\eta^{(0)}_s$ is finitely supported  with $M(0)$ particles supported on $[\vmin,\vmax]$. By bounding quantiles of compactly-supported distributions by the endpoints of their support, using the above inequalities and plugging in the definitions of $\er, \eac$ as well as using Proposition~\ref{prop:generalbound} we obtain
    \begin{align*}
        \er(k,w_1) &\leq \frac{\vmax-\vmin}{2M(k)},\;\;\eac(n,w_1) \leq \sum\nolimits_{k=1}^n \gamma^{n-k}\frac{\vmax-\vmin}{2M(k)},\;\;\overline{w}_1(\eta^{(n)},\eta^*) \leq e(n),
    \end{align*}
    where the upper bound $e(n)$ on the $n$-th approximation error is 
    \begin{equation}
        e(n)= \sum\nolimits_{k=1}^n \gamma^{n-k}\frac{[\vmax-\vmin]}{2M(k)} + \gamma^n [\vmax-\vmin].\footnote{if $M(k)\equiv m\in\bN$ is constant, $e(n)$ can be further bounded by $\left[\vmax-\vmin\right]\cdot\left[\frac{1}{2m(1-\gamma)}\;+\;\gamma^n\right]$, cf.~Lemma 5.30 of \cite{bdr2023}.}    
    \end{equation}
    Treating the sizes of $\cS, \cA$ as well as the number of particles of reward distributions as constants, the time complexity to calculate $\eta^{(n)}$ is of order $\cO(\timen)$ with 
    $$\timen = \sum\nolimits_{k=1}^nM(k)\log\left(M(k)\right),$$
    compare to page 137 in \cite{bdr2023}. To visualise the effect of choosing different size functions $M$, let 
    $$n(T) = \max\{n\in\bN_0|\timen\leq T\}$$
    be the number of iterations QDP has calculated up to (a global constant of) time $T>0$. Thus, $e(n(T))$ is a guaranteed bound for the approximation error after execution time $T$. Note that for $e(n(T))\to 0$ as $T\to\infty$ it is necessary that $M(k)\to\infty$ as $k\to\infty$. Figure~\ref{fig:qdp-time-plot} shows plots of functions $T\mapsto \log(e(n(T)))$ for different choices of sizes functions $M$: we compare $M(k)=\lceil (1/\theta)^k\rceil$ for different $\theta$ with $M(k)\equiv m$ for different $m$. The plot suggests to prefer the exponential choice with $\theta\approx \frac{\gamma+1}{2}$. 

    \begin{figure}[t]
    \centering
    \includegraphics[width=0.99\textwidth]{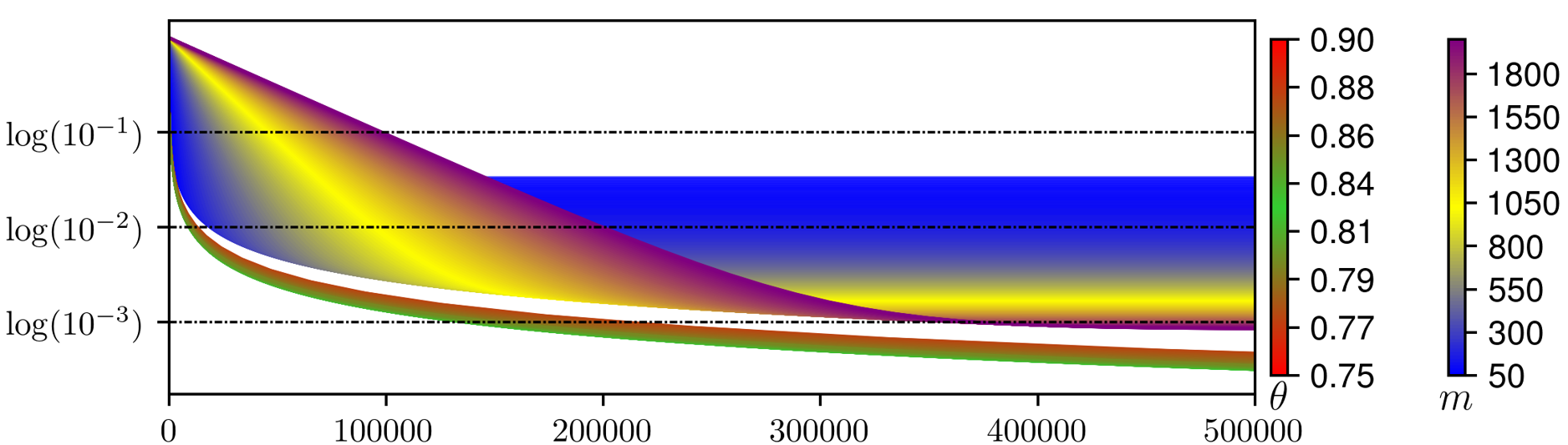} 
    \caption{(Partially) overlapping curves $T\mapsto \log(e(n(T)))$ for $\gamma=0.7$ and different size functions; $M(k)=\lceil (1/\theta)^k\rceil$ with $\theta\in[0.75,0.9]$ and $M(k)\equiv m$ constant with $m\in\{50,51,\dots,2000\}$. For each choice of $M$ the curve has a single color.}
    \label{fig:qdp-time-plot}
    \end{figure}

\end{example}

\subsection{Fixed Non-Expansive Projection}\label{sec:nonexpansion}

In case $\Xi$ contains one element, $\Pi$ can be reduced to a function $\Pi:\sPR\to\sPR$. Let $\bar\Pi:\sPRS\to\sPRS$ be the component-wise extension of $\Pi$. The $n$-th calculated approximation $\eta^{(n)}$ then takes the form $\eta^{(n)} = (\bar \Pi\cT)^{\circ n}(\eta^{(0)})$ and the parameter algorithm $A$ as well as the sequence $\bxi^{(n)}, n\in\bN$, can be eliminated from the formalism of Section~\ref{sec:framework}. This corresponds to the DDP framework considered in Chapter 5 in \cite{bdr2023}. Let $\dd$ be $c$-homogeneous, regular and convex. Proposition~\ref{prop:generalbound} applies and leads to 
\begin{equation}\label{eq:apebound}\tag{APE-bound}
    \od(\eta^{(n)},\eta^*)\leq \gamma^{cn}\od(\eta^{(0)},\eta^*) + \sum\nolimits_{k=1}^{n}\gamma^{c(n-k)}\od(\eta^{(k)},\cT\eta^{(k-1)}).
\end{equation}

In case $\Pi$ is a \emph{non-expansion} with respect to $\dd$, the approximation errors $\od(\eta^{(n)},\eta^*)$ have been studied in Chapter 5 of \cite{bdr2023}. To discuss how the bound  \eqref{eq:apebound} relates we make the following assumptions:
    \begin{itemize}
        \item[(i)] $\Pi$ is a \emph{non-expansion} with respect to $\dd$, that is $\dd(\Pi\mu,\Pi\nu)\leq \dd(\mu,\nu)$ for all $\mu,\nu\in\sPR$,
        \item[(ii)] $\cT(\sPdRS)\subseteq\sPdRS$ and $\bar\Pi(\sPdRS)\subseteq\sPdRS$,
        \item[(iii)] $\bar\Pi\cT$ has a fixed-point $\hat\eta\in\sPdRS$.
    \end{itemize}
    Note that (i) implies that $\bar\Pi\cT$ is a $\gamma^c$-contraction; (ii) implies that $\bar\Pi\cT(\sPdRS)\subseteq \sPdRS$. In particular, the fixed-point $\hat\eta$ of $\bar\Pi\cT$ is unique within $\sPdRS$. Furthermore, (i)+(ii)  imply (iii) in case the space $(\sPdRS,\od)$ is complete. By triangle inequality, conditions (i)-(iii) lead to 
    \begin{equation}\label{eq:expansionbound}\tag{$\hat\eta$-bound}
        \od(\eta^{(n)},\eta^*) \leq \gamma^{cn}\od(\eta^{(0)},\hat\eta) + \od(\hat\eta,\eta^*).
    \end{equation}
    Note that \eqref{eq:expansionbound} makes use of assumptions (i)-(iii), while \eqref{eq:apebound} does not. However, under these additional assumptions, \eqref{eq:apebound} can be further bounded and is seen to be close to \eqref{eq:expansionbound}: first, for every $k\in\bN$ it holds that 
    \begin{align*}
        \od(\eta^{(k)},\cT\eta^{(k-1)}) &\leq \od(\eta^{(k)},\hat\eta) + \od(\hat\eta,\cT\hat\eta) + \od(\cT\hat\eta,\cT\eta^{(k-1)})\\
        &\leq \gamma^{ck}\od(\eta^{(0)},\hat\eta) + \od(\hat\eta,\cT\hat\eta) + \gamma^c\od(\hat\eta,\eta^{(k-1)})\\
        &\leq 2\gamma^{ck}\od(\eta^{(0)},\hat\eta) + \od(\cT\hat\eta,\hat\eta)
    \end{align*}
    and hence
    \begin{equation}\label{eq:apebound2}
        \eqref{eq:apebound} \leq \gamma^{cn}\od(\eta^{(0)},\eta^*) + 2n\gamma^{cn}\od(\eta^{(0)},\hat\eta) + \frac{\od\left(\cT\hat\eta,\hat\eta\right)}{1-\gamma^c}.
    \end{equation}
    The only non-vanishing terms in the upper bounds \eqref{eq:expansionbound}, resp. \eqref{eq:apebound2}, are $\od(\hat\eta,\eta^*)$, resp. $\od\left(\cT\hat\eta,\hat\eta\right)/(1-\gamma^c)$. They are related by 
    \begin{equation}\label{eq:fixedpointbound}
        \frac{\max\left\{\od(\bar\Pi\eta^*,\eta^*),\od(\cT\hat\eta,\hat\eta)\right\}}{1+\gamma^c}\leq \od(\hat\eta,\eta^*)\leq \frac{\min\left\{\od(\bar\Pi\eta^*,\eta^*),\od(\cT\hat\eta,\hat\eta)\right\}}{1-\gamma^c},
    \end{equation}
    which follows from triangle inequality and contractive properties.\footnote{Also, \eqref{eq:fixedpointbound} immediately implies $\eta^*=\hat\eta\Leftrightarrow\Pi\eta^*=\eta^*\Leftrightarrow \cT\hat\eta=\hat\eta$.} Finally, 
    \begin{align*}
        \eqref{eq:expansionbound} &&\leq&&  && && &\gamma^{cn}\od(\eta^{(0)},\hat\eta) &&+&& &\frac{\od\left(\bar\Pi\eta^*,\eta^*\right)}{1-\gamma^c}&&,\\
        \eqref{eq:apebound} &&\leq&& \gamma^{cn}\od(\eta^{(0)},\eta^*) &&+&& 2n&\gamma^{cn}\od(\eta^{(0)},\hat\eta) &&+&& \frac{1+\gamma^c}{1-\gamma^c}\cdot&\frac{\od\left(\bar\Pi\eta^*,\eta^*\right)}{1-\gamma^c}&&.
    \end{align*}

    The vanishing term in the final upper bound for \eqref{eq:apebound} decays to zero faster than $\theta^n$ for every $\theta\in(\gamma^c,1)$ and the non-vanishing term is the same as in \eqref{eq:expansionbound} up to a factor $(1+\gamma^c)/(1-\gamma^c)$. 
    
   In summary, we conclude that bounding approximation errors by accumulated projection errors is a very flexible approach, that can also cover varying projections or fixed-projections that are expansive with respect to $\dd$. In the special case of fixed non-expansive projections, \eqref{eq:expansionbound} appears to be sharper than \eqref{eq:apebound}, but for practical reasons the bounds are close.

    \begin{remark}
        The QDP-Algorithm with a constant size function $M(k)\equiv m$ fits the framework of a fixed-projection, $\Pi = \Pi(\cdot,m)$. Note that this projection is an expansion with respect to $w_1$, that is the approximation errors cannot be bounded by \eqref{eq:expansionbound} directly; cf.~Lemma 5.30 and Exercise~5.20 in \cite{bdr2023}.
    \end{remark}

\section{A class of DDP algorithms}\label{sec:ddp-algorithms}

A parameterised projection $\Pi:\sPR\times\Xi\to\sPRC$ that is well-suited for DDP algorithms under (A2), closely related to \emph{quantization schemes} of \cite{lloyd1982least}, is given as follows: The parameter space is $\Xi = \bigcup\nolimits_{m\in\bN}\Xi_m$ with 
$$\TNm = \Big\{\xi=(x_1,\dots,x_m,y_1,\dots,y_{m-1})\in\bR^{2m-1}\;\Big|\;x_1\leq y_1\leq x_2\leq y_2\leq \cdots \leq y_{m-1}\leq x_m\Big\}.$$
and $\Pi(\mu,\xi)\in\sPRC$ is defined, for $\mu\in\sPR, \xi=(x_1,\dots,x_m,y_1,\dots,y_{m-1})\in\Xi_m$, as  
\begin{equation}\label{eq:pi-def}
    \Pi(\mu,\xi) = \sum\nolimits_{i=1}^m(F_{\mu}(y_i)-F_{\mu}(y_{i-1}))\cdot\delta_{x_i},
\end{equation}
with the convention $y_0=-\infty, y_m=+\infty$ and $F_{\mu}(-\infty)=0, F_{\mu}(+\infty)=1$, see Figure~\ref{fig:pi}. In particular, $\Pi(\mu,\xi)$ depends on $\mu$ only by a finite number of evaluations $F_{\mu}$, which makes it well-suited for DDP algorithms under (A2), as the CDF of $\mu=\cT_s\eta$ can be evaluated in this case by \eqref{eq:Teta-cdf}. In particular, formula \eqref{eq:Teta-cdf} directly yields the following: 

\begin{theorem}\label{thm:numericalddp}
    Let $m',m\in\bN$ and $\eta = [\eta_{\sn}:\sn\in\cS]\in\sP(\bR)^{\cS}$ with $\eta_{\sn} = \sum\nolimits_{j=1}^{m'}q_{\sn j}\delta_{z_{\sn j}}$. Let $\xi = (x_1,\dots,x_n,y_1,\dots,y_{m-1})\in\TNm$ and $s\in\cS$. Then $\PN(\cT_s(\eta),\xi) = \sum_{i=1}^m o_i\delta_{x_i}$ with
    $$o_i = \sum\nolimits_{(a,\sn)\in\cA\times\cS}\sum\nolimits_{j=1}^{m'}\pi(a|s)\cdot p(\sn|s,a)\cdot q_{\sn j}\cdot \left[F_{sa\sn}(y_i-\gamma z_{\sn j})-F_{sa\sn}(y_{i-1}-\gamma z_{\sn j})\right].$$
    That is, $\Pi(\cT_s(\eta),\xi)$ can be calculated under (A2)+(A3) in time of order $\cO\left(m'\cdot m\right)$.
\end{theorem}

\begin{figure}[t]
    \centering
    \includegraphics[width=0.99\textwidth]{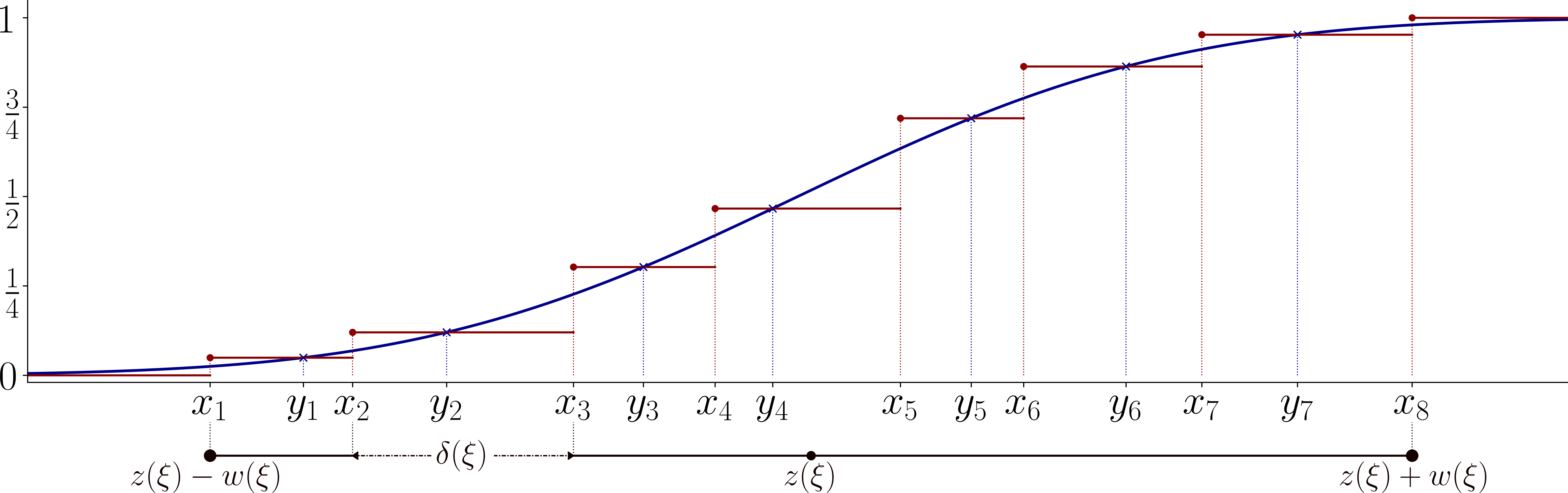} 
    \caption{CDFs of $\mu$ (blue) and $\Pi(\mu,\xi)$ (red) with $\xi=(x_1,\dots,x_8,y_1,\dots,y_7)\in\Xi_8$. The support of $\Pi(\mu,\xi)$ is the compact interval $[z(\xi)-w(\xi),z(\xi)+w(\xi)] = [x_1,x_8]$ and $\delta(\xi) = \max_{2\leq i\leq 8}|x_i-x_{i-1}|$.}
    \label{fig:pi}
\end{figure}

Let $\eta^{(0)}\in\sPRC^{\cS}$ be an initial approximation. In order to calculate approximations $\eta^{(n)}, n\in\bN$ using $\Pi$ one has to decide for a suitable \emph{parameter algorithm} 
$$A:\sPRS\times\cS\times\bN\;\longrightarrow\;\Xi = \cup_{m\in\bN}\;\Xi_m.$$
In order to control the space-time complexity, it is reasonable to choose a \emph{size function} $M:\bN\to\bN$ and to consider only parameter algorithms $A$ such that 
\begin{equation}\label{eq:size-req}
    A(\cdot,\cdot,k)\in\TN_{M(k)}\;\;\text{for all}\;\;k\in\bN.
\end{equation}
Given such a parameter algorithm, the iterative procedure to calculate $\eta^{(n)}, n\in\bN$ is thus
\begin{enumerate}
    \item Calculate $\bxi^{(n)} = [\bxi^{(n)}_s:s\in\cS]$ with $\bxi^{(n)}_s = A(\eta^{(n-1)},s,n)\in\Xi_{M(n)}$,
    \item Calculate $\eta^{(n)} = [\eta^{(n)}_s:s\in\cS]$ with $\eta^{(n)}_s = \Pi(\cT_s\eta^{(n-1)},\xi^{(n)}_s)$ as in Theorem~\ref{thm:numericalddp}.
\end{enumerate}

For practical purposes, $A$ has to be such that $A(\eta^{(n-1)},s,n)$ can be evaluated under Assumptions (A2)+(A3) within a controllable space-time complexity. In Remark~\ref{rem:minimisation-wb} we discuss the topic of \emph{optimal parameter algorithms} and demonstrate that these are, in general, \emph{not easily available for practical purposes}. In the following subsections, we consider three different parameter algorithms $A$ that can be applied under Assumptions (A2)+(A3). 
\begin{itemize}
    \item[(i)] In Section~\ref{sec:plain} we introduce plain parameter algorithms (PPA). By analysing the trade-off between approximation quality and time complexity in detail, we collect additional evidence that constructing DDP algorithms with projection errors decaying of the order $\cO(\theta^k)$ with $\theta\in[\gamma^c,1)$ seems reasonable.
    \item[(ii)] In Section~\ref{sec:plain-adaptive} we introduce an adaptive version of the plain parameter algorithms (ADP), which is derived by bounding quantiles of convolutions. A partial analysis leads to a useful \emph{black-box algorithm}, which is suggested for application in case (A1.P$\alpha$) holds for large $\alpha$.
    \item[(iii)] In Section~\ref{sec:quantile} we modify the algorithm of Section~\ref{sec:plain-adaptive} further by considering \emph{spline interpolations of inverse quantile functions} (QSP). The emerging algorithm can be seen as an approximation of the QDP-algorithm applicable under assumptions (A2)+(A3) and is a second \emph{black-box} algorithm, that can be used also in presence of heavy-tailed reward distributions, see assumption (A1.P$\alpha$).
\end{itemize}

The mathematical analysis of Section~\ref{sec:plain} relies on bounding projection errors, we present the basic tool now. Let $\xi = (x_1,\dots,x_m,y_1,\dots,y_{m-1})\in\TN$. Key-characteristics of $\xi$ are 
\begin{align*}
    \delta(\xi)=\max_{2\leq i\leq m}|x_i-x_{i-1}|,\;\;z(\xi)=\frac{x_m+x_1}{2},\;\;w(\xi) = \frac{x_m-x_1}{2}.
\end{align*}
So, for any $\mu\in\sPR$ we have that $\Pi(\mu,\xi)$ is finitely supported on $m$ points that lie in the compact interval $[z(\xi)-w(\xi),z(\xi)+w(\xi)]$, see Figure~\ref{fig:pi}. As we deal with distributions that can have unbounded support and the analysis metrics we consider (also) depend on the tail behaviour of distributions, we need to introduce the following terms: for $\mu=\cL(X)\in\sPR, \xi\in\Xi$ and $\beta\in(0,\infty)$ define 
\begin{align*}
    \tailwb(\mu,\xi) &= \int_0^{\infty}x^{\beta-1}\bP[|X-z(\xi)|>w(\xi)+x]\;\md x,\\
    \taillb(\mu,\xi) &= \int_0^{\infty}\bP[|X-z(\xi)|>w(\xi)+x]^{\beta}\;\md x.
\end{align*}

The following is proved in Appendix \ref{app:bounds}: 

\begin{lemma}\label{lemma:basic-bound}
    Let $\mu\in\sPR, \xi\in\Xi$. Then 
    \begin{itemize}
        \item[(a)] $\wb(\Pi(\mu,\xi),\mu) \leq 4\cdot \max\left\{\delta(\xi)^{\frac{\beta}{1\vee\beta}},\;\tailwb(\mu,\xi)^{\frac{1}{1\vee\beta}}\right\}$ for all $\beta\in(0,\infty)$,
        \item[(b)] $\lb(\Pi(\mu,\xi),\mu) \leq 4\cdot \max\left\{\delta(\xi)^{\frac{1}{\beta}},\;\taillb(\mu,\xi)^{\frac{1}{\beta}}\right\}$ for all $\beta\in[1,\infty)$.
    \end{itemize}
\end{lemma}

\begin{remark}[Optimal Parameter Algorithms]\label{rem:minimisation-wb}
    Let $\dd$ be a $c$-homogeneous, regular, convex analysis metric and consider only parameter algorithms that satisfy \eqref{eq:size-req}. Choosing a parameter algorithm $A$ with the property     
    \begin{equation*}
        A(\eta,s,k)\;\in\;\argmin\nolimits_{\xi\in\TN_{m}}\dd\left(\Pi(\mu,\xi),\mu\right)\;\text{with}\;(\mu,m)=(\cT_s\eta,M(k))
    \end{equation*}
    for all $\eta, s, k$, would result in calculating approximations with minimal projections errors with respect to $\dd$. The task to minimise $\dd\left(\Pi(\mu,\xi),\mu\right)$ over $\xi\in\Xi_m$ for given $\mu\in\sP(\bR)$ can, in general, only be solved by approximation methods: for example, consider $\dd=w_2$. For $\xi=(x_1,\dots,x_m,y_1,\dots,y_{m-1})\in\Xi_m$ it is  
    $$w_2(\Pi(\mu,\xi),\mu) = \sum\nolimits_{i=1}^m\int_{y_{i-1}}^{y_i}(x-x_i)^2\md \mu(x).$$
    Minimising this expression in $\xi$ is the topic of the groundbreaking paper \cite{lloyd1982least}, which has led to what is now known as \emph{Lloyd's algorithm}. This had motivated various other algorithms for related minimisation problems (involving distributions on higher dimensions, $\mu\in\sP(\bR^k)$) such as (variants of) \emph{$m$-means clustering algorithms} or algorithms to construct \emph{optimal centroidal Voronoi tessellations}. We refer to \cite{du1999centroidal} and \cite{kloeckner2012approximation}, where these types of problems are discussed also for other metrics than $\dd=w_2$ and in higher dimensions. Regarding algorithms, Lloyd's algorithm (and its variants) may depend on $\mu$ in a more complex way than on finitely many evaluations of the cdf $F_{\mu}$. Hence, using these in our situation for $\mu=\cT_s\eta$ under (A2), would need to apply further approximation methods. Considering the trade-off between space-time complexity and approximation quality, we leave it for future research to investigate if \emph{parameter algorithms} of such high complexity are of any benefit for the DDP task at hand. 

    Finally, note that the minimisation problem $\min_{\nu}w_1(\nu,\mu)$ in which the minimum runs over all $\nu$ of the form $\nu=\frac{1}{m}\sum_{i=1}^m\delta_{x_i}$ has been solved explicitly by letting $x_i=F_{\mu}^{-1}((2i-1)/2m)$, see Proposition 5.15 \cite{bdr2023}, which was used to motivate the QDP algorithm (Example~\ref{ex:qdp}) in that reference.  
\end{remark}

\subsection{Plain parameter algorithm (PPA)}\label{sec:plain}

For each $m\in\bN, m\geq 2, w>0, z\in\bR$ let 
\begin{align*}
    \xilin(m,w,z)&=(x_1,\dots,x_m,y_1,\dots,y_{m-1})\in\TNm\;\text{with}\\
    x_i &= z + \frac{2i-1-m}{m-1}\cdot w,\\
    y_i &= z + \frac{2i-m}{m-1}\cdot w,
\end{align*}
so the interlaced points $(x_1,y_1,x_2,y_2,\dots,y_{m-1},x_m)$ form an evenly spaced interpolation of the compact interval $[z-w,z+w]$. Further, with $\xi=\xilin(m,w,z)$ it holds that $(z(\xi),w(\xi),\delta(\xi))=(z, w, \frac{2w}{m-1})$. In case $m=1$ let $\xilin(1,w,z) = (z)\in\Xi_1$.
The plain parameter algorithm is defined as follows:

\begin{definition}[PPA]\label{def:plain}
    The \emph{plain parameter algorithm} (PPA) has hyper-parameter (functions) $M:\bN\to\bN,\;W:\bN\to(0,\infty),\;z\in\bR$
    with $M, W$ non-decreasing and is defined as $A(\eta,s,k) := \xilin(M(k),W(k),z)$.
\end{definition}

In the following, let $A$ be as in Definition~\ref{def:plain} and $\eta^{(0)}\in\sPRCS$ with $\eta^{(0)}_s = \delta_z$ for all $s\in\cS$. Set $M(0):=1$. The time to calculate $\eta^{(n)}$, denoted by $\timen$, is dominated by the applications of $\Pi$, thus   
\begin{equation*}
    \timen = \cO\left(\sum\nolimits_{k=1}^nM(k-1)M(k)\right).
\end{equation*}
Our goal is to analyse the trade-off between time-complexity and approximation errors with respect to metrics $\dd=\wb,\beta\in(0,\infty)$ and $\dd=\lb,\beta\in[1,\infty)$. For that, let 
$$\delta(k) = \delta(\xilin(M(k),W(k),z)) = \frac{2W(k)}{M(k)-1}.$$
In case (A1.B), that is reward distributions have uniformly bounded support $[\rmin,\rmax]$, the condition
\begin{equation}\label{eq:inclusion}
    \left[\frac{\rmin}{1-\gamma},\frac{\rmax}{1-\gamma}\right] \subseteq \left[z-W(k),z+W(k)\right]
\end{equation}
implies that all components of $\eta^*, \eta^{(n)}, n\in\bN$ are supported on $[z-W(k),z+W(k)]$ and Lemma~\ref{lemma:basic-bound} yields $\er(\wb,k) \leq 4\delta(k)^{\frac{\beta}{1\vee\beta}}$ resp. $\er(\lb,k)\leq 4\delta(k)^{\frac{1}{\beta}}$. A natural choice for $z, W$ to ensure \eqref{eq:inclusion} for all $k$ is $z = \frac{\rmax+\rmin}{2(1-\gamma)}$ and $W(k)\equiv \frac{\rmax-\rmin}{2(1-\gamma)}$ (constant). 

In presence of unbounded reward distributions, an analysis of the projection errors via Lemma~\ref{lemma:basic-bound} is still possible, but the tail bound terms need more care: in the following, $\cO$-notation is with respect to $k\to\infty$; the constant may depend on all quantities except $k$ and concrete $\mathrm{O}$-constants are presented in Appendix \ref{app:bounds}, see Theorem~\ref{thm:bounds}. 

\begin{theorem}\label{thm:bounds-new}
    \begin{itemize}
        \item[(a)] It is $\er(\wb,k) = \cO\left(\max\left\{\delta(k)^{\frac{\beta}{1\vee\beta}},\;T(k)\right\}\right)$
    with 
    \begin{equation*}
        T(k) = \begin{cases}
            W(k)^{-\frac{(\alpha-\beta)}{1\vee\beta}},&\text{if $(A1.P\alpha)$ holds with $\alpha>\beta$},\\
    \exp\left(-\frac{\lambda(1-\gamma)}{1\vee\beta}W(k)\right),&\text{if $(A1.E\lambda)$ holds with $\lambda>0$},\\
    0,&\text{if $(A1.B)$ and \eqref{eq:inclusion} hold}.
        \end{cases}
    \end{equation*}
    \item[(b)] It is $\er(\lb,k) = \cO\left(\max\left\{\delta(k)^{\frac{1}{\beta}},\;T(k)\right\}\right)$
    with 
    \begin{equation*}
        T(k) = \begin{cases}
            W(k)^{-(\alpha-\frac{1}{\beta})},&\text{if $(A1.P\alpha)$ holds with $\alpha>1/\beta$},\\
    \exp\left(-\lambda(1-\gamma)W(k)\right),&\text{if $(A1.E\lambda)$ holds with $\lambda>0$},\\
    0,&\text{if $(A1.B)$ and \eqref{eq:inclusion} hold}.
        \end{cases}
    \end{equation*}
    \end{itemize}
\end{theorem}

In the following, we only consider $\dd = \wb$ with some fixed $\beta\in(0,\infty)$. Let
$$\neps = \min\{n\in\bN\;|\;\owb(\eta^{(n)},\eta^*)\leq\varepsilon\},\;\;\varepsilon>0.$$
The following result investigates the asymptotic behaviour of $\timeeps$ as $\varepsilon\to 0$ and leads to suggest that letting the size function $M$ increase exponentially (with rate $1/\theta, \theta\in(\gamma^c,1)$, case (b)), is superior to letting it grow polynomial (case (a)). We write $f(n) = \Theta(g(n))$ if there are constants $C_1, C_2>0$ such that $C_1 g(n)\leq f(n)\leq C_2 g(n)$ for all $n\in\bN$. 

\begin{theorem}\label{thm:plain-analysis}
    Assume $(A1.P\alpha)$ holds with $\alpha>\beta$ and let $h=\frac{1\vee\beta}{\beta}\frac{\alpha}{\alpha-\beta}>1$.
    \begin{itemize}
        \item[(a)] Let $r>0$. Then 
        \begin{align*}
        \begin{rcases}\begin{dcases}
            M(n) &= \Theta\left(n^{hr}\right),\\
            W(n) &= \Theta\left(n^{\frac{\beta}{\alpha}hr}\right)
        \end{dcases}\end{rcases}
        &\Longrightarrow\;\begin{rcases}\begin{dcases}
            \owb(\eta^{(n)},\eta^*) &= \cO\left(n^{-r}\right)\;\text{as $n\to\infty$},\\
            \neps &= \cO\left((1/\varepsilon)^{1/r}\right)\;\text{as $\varepsilon\to 0$},\\
            \timen &= \cO\left(n^{2hr+1}\right)\;\text{as $n\to\infty$},\\
            \timeeps &= \cO\left((1/\varepsilon)^{2h+\frac{1}{r}}\right)\;\text{as $\varepsilon\to 0$}.
        \end{dcases}\end{rcases}\end{align*}
        \item[(b)] Let $\theta\in(\gamma^c,1)$ with $c=\min\{1,\beta\}$. Then 
        \begin{align*}
        \begin{rcases}\begin{dcases}
            M(n) &= \Theta\left((1/\theta)^{hn}\right),\\
            W(n) &= \Theta\left((1/\theta)^{\frac{\beta}{\alpha}hn}\right)
        \end{dcases}\end{rcases}
        &\Longrightarrow\;\begin{rcases}\begin{dcases}
            \owb(\eta^{(n)},\eta^*) &= \cO\left(\theta^n\right)\;\text{as $n\to\infty$},\\
            \neps &\leq \frac{\log(\epsilon)}{\log(\theta)} +\cO(1)\;\text{as $\varepsilon\to 0$},\\
            \timen &= \cO\left((1/\theta)^{2hn}\right)\;\text{as $n\to\infty$},\\
            \timeeps &= \cO\left((1/\varepsilon)^{2h}\right)\;\text{as $\varepsilon\to 0$}.
        \end{dcases}\end{rcases}
    \end{align*}
    \end{itemize}
\end{theorem}

\begin{proof}
    By Theorem~\ref{thm:bounds-new}(a) we have
    $$\er(\wb,k) = \cO\left(\max\left\{\left(\frac{W(k)}{M(k)}\right)^{\frac{\beta}{1\vee\beta}},\;W(k)^{-\frac{(\alpha-\beta)}{1\vee\beta}}\right\}\right).$$
    Choosing $M(k), W(k)$ as in (a), resp.~(b) leads to $\er(k,\wb) = \cO\left(k^{-r}\right)$, resp.~$\er(k,\wb) = \cO\left(\theta^k\right)$. Hence, by Lemma~\ref{lemma:ape-bound},  $\owb(\eta^{(n)},\eta^*) = \cO(n^{-r})$, resp.~$\cO(\theta^n)$. This immediately leads to the asymptotic bounds for $\neps$ in both cases. The time complexity bounds follow from 
    \begin{equation*}
        \sum\nolimits_{k=1}^n (k-1)^{hr}k^{hr} = \Theta(n^{2hr+1}),\;\;\sum\nolimits_{k=1}^n (1/\theta^h)^{k-1}(1/\theta^h)^k = \Theta\left((1/\theta^h)^{2n}\right),
    \end{equation*}
    both as $n\to\infty$. Plugging in the obtained asymptotic bounds for $\neps$ in the time bound yields the claim.
\end{proof}

Comparing the asymptotic expressions for $\timeeps$, we find that choosing $M, W$ as in (b), that is enforcing exponential decay with rate $\theta\in(\gamma^c,1)$ in the projection errors, leads to a slower asymptotic growth of $\timeeps$ as $\varepsilon\to 0$ when compared to enforcing polynomial decay as in (a). This is evidence for favouring exponential decay over polynomial decay. When considering $M, W$ as in (b) but with $\theta\in(0,\gamma^c)$, the same line of reasoning applies, but requiring $\timeeps=\cO((1/\varepsilon)^{h\log(\theta)/\log(\gamma^c)})$, which grows even faster then choosing $M, W$ as in (a). The case $\theta=\gamma^c$ remains open as well as the question which $\theta\in[\gamma^c,1)$ may be the preferred choice. Experiments, also see Example~\ref{ex:qdp-time}, suggest the heuristic $\theta \approx \frac{\gamma +1}{2}$ (when bounding with a $c=1$ homogeneous analysis metric). 

\begin{remark}
    Similar conclusions can be obtained when replacing the analysis metric $\dd=\wb$ with $\dd=\lb, \beta\in[1,\infty)$ and/or strengthening the moment assumption (A1.P$\alpha$) to (A1.E$\lambda$) or even (A1.B). In all of these cases, using the corresponding parts of Theorem~\ref{thm:bounds-new}, it is possible to choose $M, W$ to enforce either polynomial or exponential decay of projection errors. A verbatim analysis of $\timeeps$ always leads to the same conclusion: constructing algorithms such that projection errors decay exponentially with rate $\theta\in(\gamma^c,1)$ seems to be preferable over polynomial decay.
\end{remark}

We investigated the plain parameter algorithm to collect further evidence that choosing $M(k) = \lceil(1/\theta)^k\rceil$ with $\theta\approx \frac{\gamma+1}{2}$ is a reasonable black-box approach. However, for practical purposes, the plain parameter algorithms is problematic: the $k$-th update step projects $\cT_s\eta^{(k-1)}$ to a distribution supported on $[z-W(k),z+W(k)]$. If the parameters $W, z$ are chosen poorly, it may occur that, for reasonably small $k$, the interval $[z-W(k),z+W(k)]$ carries only a (very) small amount of mass of $\cT_s\eta^{(k-1)}$ which leads to a high projection error. This is (partly) resolved by a modification of the plain parameter algorithm discussed in the next section.

\subsection{Adaptive Plain Parameter Algorithm}\label{sec:plain-adaptive}

For practical purposes, we propose a modification of the plain parameter algorithm by replacing $[z-W(k),z+W(k)]$ with a compact interval that contains a \emph{guaranteed} amount of mass of $\cT_s\eta^{(k-1)}$. This modification is derived and justified by the following Lemma, shown in Appendix \ref{app:range}:
\begin{lemma}\label{lemma:range}
    Let $\eta = [\eta_{\sn}:\sn\in\cS]\in\sPRS, \varepsilon_u\in(0,\nicefrac{1}{2}]$ and $s\in\cS$. Define
    \begin{align}
        \xmin &= \min\limits_{(a,\sn)}\Big\{F_{sa\sn}^{-1}\!\left(1-\sqrt{1-\varepsilon_u}\right)\!+\!\gamma\!\cdot\! F_{\eta_{\sn}}^{-1}\!\!\left(1-\sqrt{1-\varepsilon_u}\right)\Big\}\label{eq:xmin},\\
        \xmax &= \max\limits_{(a,\sn)}\Big\{F_{sa\sn}^{-1}\!\left(\sqrt{1-\varepsilon_u}\right)\!+\!\gamma\!\cdot\! F_{\eta_{\sn}}^{-1}\!\!\left(\sqrt{1-\varepsilon_u}\right)\Big\}\label{eq:xmax},
    \end{align}
    where both $\min$ and $\max$ run over all pairs $(a,\sn)\in\cA\times\cS$ with $\pi(a|s)p(\sn|s,a)>0$. Then  
    \begin{equation*}
        \max\Big\{\cT_s\eta(-\infty,\xmin),\;\cT_s\eta(\xmax,\infty)\Big\} \leq \varepsilon_u\;\;\;\text{and}\;\;\;\cT_s\eta[\xmin,\xmax] \geq 1-2\varepsilon_u.
    \end{equation*}
\end{lemma}

As the QF of a given finitely supported distribution can be evaluated, the values $\xmin$ and $\xmax$ can be calculated under Assumptions (A2)+(A3) for any $\eta\in\sPRCS$. This leads to the following:

\begin{definition}[ADP]\label{def:plain-adap}
    The \emph{adaptive plain parameter algorithm (ADP)} has hyper-parameter functions  $M:\bN\to\bN,\;\mathcal{E}_u:\bN\to(0,\nicefrac{1}{2}]$
    with $M$ non-decreasing and $\mathcal{E}_u$ non-increasing and calculates $A(\eta,s,k)\in \Xi_{M(k)}$ as follows: 
    \begin{enumerate}
        \item $(m,\varepsilon_u) = \left(M(k), \mathcal{E}_u(k)\right)$,
        \item Calculate $\xmin$ and $\xmax$ as in \eqref{eq:xmin} and \eqref{eq:xmax},
        \item $A(\eta,s,k) = \xilin(m, \frac{\xmax-\xmin}{2}, \frac{\xmax+\xmin}{2})$.
    \end{enumerate}
\end{definition}

For practical purposes, we suggest to use (ADP) only if assumption \emph{(A1.P$\alpha$) is satisfied  with a large~$\alpha$}, in which case we recommend the hyper-parameter functions 

\begin{equation}\label{eq:choice}
    M(k) = \lceil (1/\theta)^k\rceil,\;\;\mathcal{E}_u(k) = \frac{1}{2 M(k)}\;\;\text{with}\;\theta = \frac{\gamma+1}{2}.
\end{equation}

In case (A1.P$\alpha$) does not hold for some large $\alpha$, that is in the presence of heavy-tailed reward distributions (where we recommend to use the QSP algorithm discussed in Section \ref{sec:quantile}), the discussion in Appendix \ref{app_hyper_choice} leads to suggest that this choice of hyper-parameters may not yield high quality approximations, which is in line with the results of our controlled experiment in Section~\ref{sec:simulation}.  The reasoning behind recommendation \eqref{eq:choice} is discussed in Appendix \ref{app_hyper_choice}. 

\subsection{Quantile-Spline Parameter Algorithm}\label{sec:quantile}

Choosing a parameter algorithm $A$ that produces evenly spaced interpolation points is not necessary to obtain high quality approximations. Aiming to approximate QDP, we now modify the adaptive plain parameter algorithm of Section~\ref{sec:plain-adaptive} leading to a DDP algorithm that is applicable under (A2) and seems suitable also in presence of heavy-tailed reward distributions.  

Recall that QDP updates $\eta\in\sPRS$ using $m\in\bN$ particles as $\tilde\eta_s = \frac{1}{m}\sum\nolimits_{i=1}^{m}\delta_{x_i}$ with $x_i = F_{\cT_s\eta}^{-1}\left(\frac{2i-1}{2m}\right)$. In case $\cT_s\eta$ is continuous, we have that $\tilde\eta_s = \Pi(\cT_s\eta,\tilde\xi)$ with $\tilde\xi=(x_1,\dots,x_{m},y_1,\dots,y_{m-1})\in\Xi_m$ where $x_i$ is as before and $y_i = F_{\cT_s\eta}^{-1}\left(\frac{i}{m}\right)$. As explained above, calculating quantiles of $\cT_s\eta$ is problematic under the assumption (A2). We suggest to replace $\tilde\xi \in \Xi_{m}$ with an \emph{approximation of quantiles} $\xi\in\Xi_m$ that uses only a controllable number of evaluations of the cdf $F_{\cT_s\eta}$. Note that the approximation we suggest takes place on the level of the \emph{parameters}: errors in the quantile approximation $\xi\approx \tilde \xi$ may result in a non-optimal parameter, but by calculating $\Pi(\cT_s\eta,\xi)$ the approximated quantiles are \emph{reweighed} according to the true distribution $\cT_s\eta$. 

The suggested approximation of quantiles of $\cT_s\eta$ is constructed as follows: first, we calculate the interval $[\xmin,\xmax]\subset \bR$ as in \eqref{eq:xmin} and \eqref{eq:xmax} with $\mathcal{E}_u(n) = \frac{1}{2m}$. Thus, this interval contains at least a mass of $1-\frac{1}{m}$ of $\cT_s\eta$. On that interval, we perform a linear interpolation of $F_{\cT_s\eta}$ via $m'\in\bN$ evaluations on evenly spaced points in $[\xmin, \xmax]$. This interpolated CDF is continuous and can easily be inverted; we define $\xi$ to contain quantiles of this interpolated CDF. Linear interpolation and inversion can be combined in one step, which is step $4$ in the following: 

\begin{definition}[QSP]\label{def:quantile-adap}
    The hyper-parameters are $M:\bN\to\bN,\;M':\bN\to\bN$
    with $M, M'$ non-decreasing. The \emph{quantile-spline parameter algorithm (QSP)} calculates $A(\eta,s,k)\in \Xi_{M(k)}$ as follows: 
    \begin{enumerate}
        \item $(m,m') = \left(M(k), M'(k)\right)$,
        \item Calculate $\xmin,\xmax$ as in \eqref{eq:xmin} and \eqref{eq:xmax} with $\epsilon_u = \frac{1}{2m}$.
        \item Let $z_l \leftarrow \xmin + (\xmax-\xmin)\!\cdot\!\frac{l}{m'+1}$ for $l=0,1,\dots,m'+1$,
        \item Let $L:[0,1]\to\bR$ be the linear spline through the points\footnote{remove all pairs $(p_l,z_l)$ for which there is $l'\in\{0,\dots,l-1\}$ with $p_{l'}=p_l$.}
        $$\left(0,z_0\right),\;\left(F_{\cT_s\eta}(z_1),z_1\right),\;\left(F_{\cT_s\eta}(z_2),z_2\right),\dots,\;\left(F_{\cT_s\eta}(z_{m'}),z_{m'}\right),\;\left(1,z_{m'+1}\right),$$
        \item $A(\eta,s,k) = (x_1,\dots,x_m,y_1,\dots,y_{m-1})$ with $x_i = L\big(\frac{i-1}{m-1}\big)$ and $y_i = L\big(\frac{2i-1}{2m-2}\big)$.
    \end{enumerate}
\end{definition}

We suggest to choose the hyper-parameter functions $M, M'$ as follows:
\begin{equation}\label{eq:choice2}
    M(k)=\lceil (1/\theta)^k\rceil,\;\;M'(k) = \lceil (1/4)\cdot(1/\theta)^k\rceil\;\;\text{with}\;\theta = \frac{\gamma+1}{2}.
\end{equation}

\subsection{Controlled Experiment}\label{sec:simulation}

We consider two versions of a MDP in which return distributions are known analytically: let $\cA = \{a\}$ have one element and $\cS=\{1,2,3\}$, thus $\pi(a|s)\equiv 1$. State-action-state transitions are deterministic and circular: $p(\sn|s,a) = 1:\Leftrightarrow (s,\sn)\in\{(1,2),(2,3),(3,1)\}$. Besides $s$ and $\gamma:=0.7$ the return distribution $\eta^*_s$ depends only on the three reward distributions $\cL(R_{1a2})$, $\cL(R_{2a3})$, $\cL(R_{3a1})$. For all $(s,\sn)$ with $(s,\sn)\notin\{(1,2),(2,3),(3,1)\}$ set $\cL(R_{sa\sn}) = \delta_0$. For the three relevant reward distributions, we consider two versions. Let $\norm(\mu,\sigma^2)$ be the normal distribution with expectation $\mu\in\bR$ and variance $\sigma^2\in(0,\infty)$ and let $\cauchy(\mu,s)$ be the Cauchy distribution with location $\mu\in\bR$ and scale $s>0$. It is $\cauchy(\mu,s)\in \sP_{\alpha}(\bR)$ for $\alpha\in(0,1)$, but $\cauchy(\mu,s)\notin\sP_1(\bR)$, while $\norm(\mu,\sigma^2)\in\sP_{\alpha}(\bR)$ for all $\alpha\in(0,\infty)$.

The two versions we consider are:  
\begin{align*}
    (i)\;\;&&\begin{pmatrix}\cL(R_{1a2}) \\ \cL(R_{2a3}) \\ \cL(R_{3a1})\end{pmatrix} &= \begin{pmatrix}\norm(-3,1) \\ \norm(5,2) \\ \norm(0,0.5)\end{pmatrix} &&\Longrightarrow &\begin{pmatrix}\cL(\eta^*_1) \\ \cL(\eta^*_2) \\ \cL(\eta^*_3)\end{pmatrix} &= \begin{pmatrix}\norm(0.761,2.380) \\ \norm(5.373,2.816) \\ \norm(0.533,1.666)\end{pmatrix},\\
    (ii)\;\;&&\begin{pmatrix}\cL(R_{1a2}) \\ \cL(R_{2a3}) \\ \cL(R_{3a1})\end{pmatrix} &= \begin{pmatrix}\cauchy(-3,0.5) \\ \cauchy(5,0.1) \\ \cauchy(0,5)\end{pmatrix} &&\Longrightarrow &\begin{pmatrix}\cL(\eta^*_1) \\ \cL(\eta^*_2) \\ \cL(\eta^*_3)\end{pmatrix} &= \begin{pmatrix}\cauchy(0.761,4.597) \\ \cauchy(5.373,5.852) \\ \cauchy(0.533,8.218)\end{pmatrix}.
\end{align*}

We compare three DDP algorithms to approximate (estimate) $\eta^*$: The two blackbox algorithms of Section~\ref{sec:ddp-algorithms} with parameter algorithms ADP (Section~\ref{sec:plain-adaptive}) and QSP (Section~\ref{sec:quantile}) and hyper-parameter choices as in \eqref{eq:choice} and \eqref{eq:choice2} and initial approximation $\eta^{(0)} = [\delta_0:\sn\in\cS]$. We let both algorithms run up to time $t_{\mathrm{max}} = 45$ seconds (on a standard notebook) and return the last completely calculated approximation $\eta^{(n)}$ before that time. Third, we consider a Monte-Carlo estimation (MC), for which we choose $n=30$ fixed, approximate $\eta^* \approx \cT^{\circ n}(\eta^{(0)})$ and estimate the components of the latter by using that, for large $N$, we have $\cT^{\circ n}_s(\eta^{(0)}) \approx\frac{1}{N}\sum_{l=1}^N\delta_{G_{s,n-1,l}}$, where $G_{s,n-1,l}, l=1,2,\dots$ are iid samples of $\cT^{\circ n}_s(\eta^{(0)})= \cL(\sum_{t=0}^{n-1}\gamma^t R^{(t)}|S^{(0)}=s)$. A sample of $\cT^{\circ n}_s(\eta^{(0)})$ can be generated by simulating a trajectory of the full MDP started in $s$ up to length $n$ and calculating the $\gamma$-discounted sum of rewards along the way. For each $s$ we generate $N$ such samples, where $N$ is the maximal amount that is possible within $t_{\mathrm{max}} = 45$ seconds. For comparison, we repeat the Monte Carlo procedure, but with $N$ chosen such that $3N$ (the amount of generated samples) equals the size of the output of the ADP algorithm. The output of this second Monte Carlo estimation is (MC2). 

Each procedure returns some approximation (estimation) $\eta^{(n)}$. We report $\od(\eta^{(n)},\eta^*)$ for $\dd=\ks, w_1, \ell_2$ in Figures~\ref{fig:sim-norm} and \ref{fig:sim-cauchy}. For mdp (i) it is seen that both ADP and QSP outperform (MC, MC2) significantly for any analysis metric. The approximation qualities of ADP and QSP are in a comparable range. Things differ for mdp (ii): still, QSP clearly outperforms (MC, MC2) in any metric, but now ADP fails to yield useful approximations, which is to be expected: the heuristics in choosing the hyper-parameters of ADP, see \eqref{eq:choice}, fail for heavily tailed distributions such as the Cauchy distributions. 

\begin{figure}[t]
    \centering
\def\arraystretch{1.1}
\begin{tabular}{|c|ccc|ccc|c|}
\hline
\text{alg} & \text{time(s)} & \text{size} & $n$ & $\oks$ & $\overline{w_1}$ & $\overline{\ell_2}$ & \text{type} \\
\hline
\hline
ADP & 40.17 & 38868 & 54 & 0.0003 & 0.0007 & 0.0003 & \text{numerical} \\
QSP & 34.32 & 33036 & 53 & 0.0002 & 0.0025 & 0.0005 & \text{numerical} \\
MC & 45.00 & 15588 & 30 & 0.0125 & 0.0296 & 0.0132 & \text{random} \\
MC2 & 134.11 & 38868 & 30 & 0.0079 & 0.0168 & 0.0067 & \text{random} \\
\hline
\end{tabular}
    \caption{Results for mdp (i), where size is the number of stored particles, resp. the number of stored samples in MC estimation. Calculating the next approximation of ADP and QSP exceeds $45$ seconds.}
    \label{fig:sim-norm}
\end{figure}

\begin{figure}[t]
    \centering
\def\arraystretch{1.1}
\begin{tabular}{|c|ccc|ccc|c|}
\hline
\text{alg} & \text{time(s)} & \text{size} & $n$ & $\oks$ & $\overline{w_1}$ & $\overline{\ell_2}$ & \text{type} \\
\hline
\hline
ADP & 36.96 & 38868 & 54 & 0.2317 & $\infty$ & 0.6892 & \text{numerical} \\
QSP & 34.61 & 33036 & 53 & 0.0012 & $\infty$ & 0.0399 & \text{numerical} \\
MC & 44.99 & 14643 & 30 & 0.0138 & $\infty$ & 0.0967 & \text{random} \\
MC2 & 146.22 & 38868 & 30 & 0.0080 & $\infty$ & 0.1285 & \text{random} \\
\hline
\end{tabular}
\caption{Results for mdp (ii). Since $\cauchy(\mu,s)\notin\sP_1(\bR)$, the $w_1$-distances are infinite. However, $\cauchy(\mu,s)\in \sP_{1/2}(\bR) \subsetneq \sP_{\ell_2}(\bR)$, thus $\ell_2$-distances are finite.}
    \label{fig:sim-cauchy}
\end{figure}

\section{Kolmogorov--Smirnov distance and density approximation}\label{sec:ks}

We discussed how to obtain bounds for $\od(\eta^{(n)},\eta^*)$ with respect to a $c$-homogeneous, regular and convex analysis metric $\dd$, for example $\dd=\wb, \beta\in(0,\infty]$ or $\dd=\lb,\beta\in[1,\infty)$. The Kolmogorov--Smirnov distance $\ks=\ell_{\infty}$ is not covered by these results. In Section~\ref{sec:bounds}, we show how to bound $\oks(\eta^{(n)},\eta^*)$ in terms of $\owb(\eta^{(n)},\eta^*), \beta\in(0,\infty)$ provided the components $\eta^*_s$ satisfy certain \emph{continuity properties}, such as possessing a bounded pdf $f^*_s$. Further, if a pdf $f^*_s$ for $\eta^*_s$ is known to exists, we suggest to construct an approximation of it by choosing $\delta>0$ and considering $f^{(n)}_{s,\delta} = f_{\eta^{(n)}_s,\delta}$, where, for any $\mu\in\sPR$ and $\delta>0$, the probability density $f_{\mu,\delta}:\bR\to[0,\infty)$ is defined by
\begin{equation*}
    f_{\mu,\delta}(x) = \frac{F_{\mu}(x+\delta)-F_{\mu}(x-\delta)}{2\delta},\;\;x\in\bR.
\end{equation*}
Under suitable regularity assumptions on $f^*_s$, we further show how the supremum distance between $f^{(n)}_{s,\delta}$ and $f^*_s$ can be bounded in terms of $\oks(\eta^{(n)},\eta^*)$. In Section~\ref{sec:densities}, we provide sufficient criteria for the return distribution to posses probability density functions (satisfying the needed regularity properties) that can, in principle, be checked based on the given ingredients of the mdp. 

\subsection{Uniform Bounds}\label{sec:bounds}

Let $M>0, \varrho\in(0,1]$. A function $g:\bR\to\bR$ is $\varrho$-Hölder (continuous) with constant $M$ if $|g(x)-g(y)|\leq M\cdot |x-y|^{\varrho}$ holds for all $x,y\in\bR$. Recall, that a distribution $\nu\in\sPR$ has pdf $f_{\nu}:\bR\to[0,\infty]$ if for all $x\in\bR$ we have  $F_{\nu}(x) = \int_{-\infty}^xf_{\nu}(y)\md y$. If $f_{\nu}$ is bounded by $M>0$, then $|F_{\nu}(x)-F_{\nu}(y)| \leq M |x-y|$ and thus $F_{\nu}$ is $\varrho=1$-Hölder continuous with constant $M$. 

Part (a) of the following Lemma is a slight generalisation of Lemma 5.1 in \cite{fija02}, part (b) follows from Lemma 2.7 in \cite{knne08}:

\begin{lemma}\label{lemma:ksbound}
    Let $\mu,\nu\in\sPR$. 
    \begin{itemize}
        \item[(a)] If $F_{\nu}$ is $\varrho$-Hölder with constant $M$, then for every $\beta\in(0,\infty)$
        \begin{equation}\label{eq:ksbound}
            \ks(\mu,\nu) \leq a^{-a}\cdot M^{1-a}\cdot \wb(\mu,\nu)^{a(1\vee\beta)}\;\;\text{with}\;\;a=\frac{\varrho}{\varrho+\beta}.
        \end{equation}
        In particular, if $\nu$ has a pdf $f_{\nu}$ bounded by $M$, then \eqref{eq:ksbound} holds with $\varrho=1$.
        \item[(b)] If $\nu$ has a pdf $f_{\nu}$ that is $\tau$-Hölder with constant $C$, then for every $\delta\in(0,\infty)$
        \begin{equation*}
            \sup_{x\in\bR}|f_{\mu,\delta}(x) - f_{\nu}(x)| \leq \frac{1}{\delta}\ks(\mu,\nu) + C\cdot \delta^{\tau}.
        \end{equation*}
        If, additionally, $f_{\nu}$ is bounded by $M$, then $\ks(\mu,\nu)$ in the upper bound can be further bounded by \eqref{eq:ksbound} with $\varrho=1$.
    \end{itemize}
\end{lemma}
\begin{proof}
    For (a) we note that Lemma~5.1 in \cite{fija02} is the case in which $\nu$ has a density bounded by $M$, which they  used to bound $\max_{|x-y|\leq\epsilon}|F_\nu(x)-F_{\nu}(y)| \leq M\varepsilon^{\varrho}$ with $\varrho=1$. It is easy to see that their proof extends to  general $\rho>0$. For (b), see Lemma 2.7 in \cite{knne08}.
\end{proof}

By passing to supremum distances over states, Lemma~\ref{lemma:ksbound} yields the following:

\begin{corollary}\label{corrolary:ksbounds}
    Let $\eta^{(n)} = [\eta^{(n)}_s:s\in\cS]$ be any approximation of $\eta^* = [\eta^*_s:s\in\cS]$.
    \begin{itemize}
        \item[(a)] If every $F_{\eta^*_s}$ is $\varrho$-Hölder with constant $M$, then for every $\beta\in(0,\infty)$
        \begin{equation}\label{eq:ksbound-sup}
            \oks(\eta^{(n)},\eta^{*}) \leq a^{-a}\cdot M^{1-a}\cdot \owb(\eta^{(n)},\eta^{*})^{a(1\vee\beta)}\;\;\text{with}\;\;a=\frac{\varrho}{\varrho+\beta}.
        \end{equation}
        In particular, if every $\eta^*_s$ has a pdf $f^*_{s}$ bounded by $M$, then \eqref{eq:ksbound-sup} holds with $\varrho=1$.
        \item[(b)] If every $\eta^*_s$ has a pdf $f^*_{s}$ that is $\tau$-Hölder with constant $C$, then for every $\delta\in(0,\infty)$
        \begin{equation}\label{rej_bound}
            \max_{s\in\cS}\sup_{x\in\bR}|f^{(n)}_{s,\delta}(x) - f^*_{s}(x)| \leq \frac{1}{\delta}\oks(\eta^{(n)},\eta^{*}) + C\cdot \delta^{\tau}.
        \end{equation}
        If, additionally, every $f^*_{s}$ is bounded by $M$, then $\oks(\eta^{(n)},\eta^{*})$ in the upper bound can be further bounded by \eqref{eq:ksbound-sup} with $\varrho=1$.
    \end{itemize}
\end{corollary}
\begin{remark}\label{rem_perfect_simulation}
The approximation of densities in supremum norm has also an application in nonuniform random number generation in the context of von Neumann's rejection method; for an introduction see Section II.3 of  \cite{devroye1986}. When applying the rejection method it may not be possible to evaluate a given probability density to decide about rejection resp.~acceptance of a sample, e.g., the density may only be given implicitly as the density of the fixed-point of a DBO. Then, to make the rejection method applicable it is sufficient (besides constructing an appropriate dominating density) to be able to approximate the given density to arbitrary precision in finite time as demonstrated  in \cite{define00,dene02}. Since the big-$\mathrm{O}$ constants in our Corollary \ref{corrolary:ksbounds} can be made explicit, the bound  (\ref{rej_bound}) is sufficient for the purpose of the rejection method. Note that for perfect simulation from fixed-points of the DBO also coupling from the past algorithms may, in principle, be applicable, see \cite{deja11,dane14}.

For other potential use of Corollary \ref{corrolary:ksbounds} in the context of certain statistical functionals see Chapter 20 of \cite{Vaart_1998}.
\end{remark}

\subsection{Sufficient criteria for existence of return densities with properties}\label{sec:densities}

We provide sufficient criteria for when the return distribution components possess pdfs, resp.~bounded, resp.~Hölder continuous pdfs. Let $s_0\in\cS$ and $(S^{(t)},A^{(t)},R^{(t)})_{t\in\bN_0}$ be the full MDP. We write $\bP_{s_0}$ instead of $\bP[\;\cdot\;|S^{(0)}=s_0]$, similar $\bE_{s_0}$. In particular, we have $\eta^*_{s_0} = \bP_{s_0}[G\in\cdot\;]$ with $G=\sum_{t}\gamma^t R^{(t)}$. Further, fix an arbitrary subset ${\cG\subseteq\cS\times\cA\times\cS}$ and define the $\bN_0\cup\{\infty\}$-valued random hitting time 
$$\tauG = \min\big\{t\in\bN_0\;\big|\;(S^{(t)},A^{(t)},S^{(t+1)})\in \cG\big\},$$
with $\min\emptyset:=\infty$. Note that ${\bE_{s_0}[\tauG]<\infty}$ implies $\cG\neq\emptyset$ and $\bP_{s_0}[\tauG\in\bN]=1$. 

\begin{theorem}\label{thm:density}
    Suppose $\bE_{s_0}[\tauG]<\infty$ and that for every $(s,a,\sn)\in\cG$ the distribution $r(\;\cdot\;|s,a,\sn)$ has a pdf $f_{(s,a,\sn)}$. Then $\eta^*_{s_0}$ has pdf $f^*_{s_0}$ given by 
    \begin{equation}\label{eq:pdf}
    f^*_{s_0}(x) = \bE_{s_0}\Big[\frac{1}{\gamma^{\tauG}}f_{(S,A,\bar S)}\left(\frac{x-Z}{\gamma^{\tauG}}\right)\Big],\;x\in\bR, 
    \end{equation}
    where $(S,A,\bar S) = (S^{(\tauG)},A^{(\tauG)},S^{(\tauG+1)}) \in \cG$ and $Z = \sum_{t\in\bN_0\setminus\{\tauG\}}\gamma^tR^{(t)}$.
\end{theorem}

The proof of Theorem~\ref{thm:density} can be found in the Appendix \ref{app_proof_24}. From \eqref{eq:pdf} one can deduce sufficient criteria for when $\eta^*_{s_0}$ has a bounded, resp.~Hölder continuous pdf. Let $\Psi(z) = \bE_{s_0}[e^{z\tauG}], z\in\bR$ be the moment generating function of $\tauG$. Note that we may have $\Psi(z)=\infty$ and that $\Psi(z)<\infty$ for some $z>0$ implies $\bE_{s_0}[\tauG]<\infty$.

\begin{corollary}\label{cor:density}
    In the setting of Theorem~\ref{thm:density}, let $f^*_{s_0}$ be the pdf of $\eta^*_{s_0}$ given by~\eqref{eq:pdf}.
    \begin{itemize}
        \item[(a)] If every $f_{(s,a,\sn)}, (s,a,\sn)\in\cG$ is bounded and $\Psi\big(\log(1/\gamma)\big)<\infty$, then $f^*_{s_0}$ is bounded.
        \item[(b)] If every $f_{(s,a,\sn)}, (s,a,\sn)\in\cG$ is $\tau$-Hölder continuous and $\Psi\big((\tau+1)\log(1/\gamma)\big)<\infty$, then $f^*_{s_0}$ is $\tau$-Hölder continuous.
    \end{itemize}
\end{corollary}
The proof of Corollary \ref{cor:density} is contained in Appendix \ref{app_proof_24}.

\begin{example}
    Consider $\cS=\{0,1\}, \cA=\{0\}, s_0=0$ and, for all $(s,a)\in\cS\times\cA$, the reward distributions $r(\;\cdot\;|s,a,0) = \delta_0$ and let $r(\;\cdot\;|s,a,1)$ have pdf $f$. Consider $\cG=\{(0,0,1),(1,0,1)\}$ and let $\alpha = p(1|0,0)\in(0,1)$. It then holds that $\bP_{s_0}[\tauG=k] = \alpha(1-\alpha)^k, k\in\bN_0$, that is $\tauG$ has a geometric distribution with parameter $\alpha$. In particular, $\bE_{s_0}[\tauG]<\infty$ and hence, by Theorem~\ref{thm:density}, $\eta^*_{s_0}$ has pdf $f^*_{s_0}$ given by \eqref{eq:pdf}, note that $f_{(S,A,\bar S)} = f$ in this case. By geometric series arguments, $\Psi(z)<\infty \Leftrightarrow e^z(1-\alpha)<1$ and hence, by Corollary~\ref{cor:density},
    \begin{itemize}
        \item $f$ bounded and $(1-\alpha)<\gamma\;\Longrightarrow\;$ $f^*_{s_0}$ is bounded,
        \item $f$ $\tau$-Hölder continuous and $(1-\alpha)<\gamma^{\tau+1}\;\Longrightarrow\;$ $f^*_{s_0}$ is $\tau$-Hölder continuous.
    \end{itemize}
\end{example}

The presented conditions for $\eta^*_{s_0}$ to have a pdf (with properties) are not necessary, as illustrated by the following example, which also shows that finding \emph{sufficient and necessary} conditions for $\eta^*_{s_0}$ to possess a pdf (with properties) is a hard problem in general, we refer to Section~2.5 of \cite{diaconis1999iterated} for discussions of related examples. 

\begin{example}
    Let $\cS=\cA=\{0\}$, $r(\;\cdot\;|0,0,0) = \frac{1}{2}\delta_0 + \frac{1}{2}\delta_1$ (fair coin toss) and $\gamma=\frac{1}{2}$. Then $\eta^*_0$ is continuous uniform on $[0,2]$, hence possesses a bounded pdf, although the reward distribution does not have a pdf. When the reward distribution is changed to $r(\;\cdot\;|0,0,0) = q\delta_0 + (1-q)\delta_1$ with $q\in(0,1), q\neq \frac{1}{2}$, then $\eta^*_{s_0}$ does not possess a pdf, although still having a Hölder continuous cdf.
\end{example}

\acks{The first author was partially supported by the Deutsche Forschungsgemeinschaft (DFG, German
Research Foundation) - 502386356.}

\appendix

\section{Appendix}

\subsection{Characterisation of \texorpdfstring{$\sP_{\lb}(\bR)$}{} for \texorpdfstring{$\beta\in[1,\infty)$}{}} \label{app:spaces}

\begin{proposition}\label{prop:lb-space}
    $\sP_{\ell_1}(\bR) = \sP_1(\bR)$ and for every $\beta\in(1,\infty)$ it holds that 
    $$\sP_{\frac{1}{\beta}}(\bR) \subsetneq \sP_{\lb}(\bR) \subsetneq \bigcap\nolimits_{0<\varepsilon<\frac{1}{\beta}}\sP_{\frac{1}{\beta}-\varepsilon}(\bR).$$
\end{proposition}
\begin{proof}
    Let $\mu=\cL(X)\in\sPR$. It is $\bE[|X|^{\alpha}] = \alpha\int_0^{\infty}x^{\alpha-1}\bP[|X|>x]\md x$ for every $\alpha\in(0,\infty)$ and $\lb(\mu,\delta_0)^{\beta} = \int_{0}^{\infty} \left(\bP[X<x]^{\beta} + \bP[X>x]^{\beta}\right)\md x$ for every $\beta\in[1,\infty)$. Further, for every $x>0$ it is $\bP[X<x]^{\beta} + \bP[X>x]^{\beta}\leq \bP[|X|>x]^{\beta} \leq 2^{\beta-1}\left(\bP[X<x]^{\beta} + \bP[X>x]^{\beta}\right)$ and thus
    \begin{equation*}
        \mu\in\sP_{\lb}(\bR) \Leftrightarrow \int_0^{\infty}\bP[|X|>x]^{\beta}\md x<\infty\;\;\text{and}\;\;\mu\in\sP_{\alpha}(\bR) \Leftrightarrow \int_0^{\infty}x^{\alpha-1}\bP[|X|>x]\md x<\infty.
    \end{equation*}
    The conditions are equivalent in case $\beta=\alpha=1$ and $\sP_{\ell_1}(\bR) = \sP_1(\bR)$ follows immediately. Now, let $\beta\in(1,\infty)$. For the first inclusion, let $\mu\in \sP_{\frac{1}{\beta}}(\bR)$. By Markov's inequality, for every $x>0$ it is $\bP[|X|>x]\leq \bE[|X|^\frac{1}{\beta}]x^{-\frac{1}{\beta}}$. Using ${\beta-1\geq 0}$ and $\frac{-(\beta-1)}{\beta} = \frac{1}{\beta}-1$ leads to
    \begin{align*}
        \int_0^{\infty}\bP[|X|>x]^{\beta}\md x &= \int_0^{\infty}\bP[|X|>x]^{\beta-1}\cdot \bP[|X|>x]\;\md x\\
        &\leq \int_0^{\infty}\left(\bE[|X|^{\frac{1}{\beta}}]x^{-\frac{1}{\beta}}\right)^{\beta-1}\cdot \bP[|X|>x]\;\md x = \beta\cdot \bE[|X|^{\frac{1}{\beta}}]^{\beta},
    \end{align*}
    which is $<\infty$ by assumption. For the second inclusion, let $\mu\in\sP_{\lb}(\bR)$ and $\varepsilon\in(0,\frac{1}{\beta})$. Using the integral formula for $\bE[|X|^{\frac{1}{\beta}-\varepsilon}]$ and applying Hölder's inequality with $p=\frac{\beta}{\beta-1}$ and $q=\beta$ gives
    \begin{align*}
        \left(\frac{1}{\beta}-\varepsilon\right)^{-1}\cdot\bE[|X|^{\frac{1}{\beta}-\varepsilon}] &= \int_0^{\infty}x^{\frac{1}{\beta}-\varepsilon-1}\cdot\bP[|X|>x]\;\md x\\
        &\leq \left[\int_0^{\infty} x^{\left[\frac{1}{\beta}-\varepsilon-1\right]\frac{\beta}{\beta-1}}\md x\right]^{\frac{\beta-1}{\beta}}\cdot\left[\int_0^{\infty} \bP[|X|>x]^{\beta}\md x\right]^{\frac{1}{\beta}}.
    \end{align*}
    The first factor is $<\infty$ because $\left[\frac{1}{\beta}-\varepsilon-1\right]\frac{\beta}{\beta-1} < -1$ and the second by assumption. 
    To see that both inclusions are strict, note that $\int_2^{\infty}x^{-1}\log(x)^{-c}\md x<\infty$ if and only if $c>1$. Thus, if $\mu\in\sPR$ has tail probabilities $\bP[|X|>x] = x^{-\frac{1}{\beta}}\log(x)^{-1}$ for $x\geq 2$ it is $\mu\in\sP_{\lb}(\bR)$, but $\mu\notin \sP_{\frac{1}{\beta}}(\bR)$. Similar, if $\bP[|X|>x] = x^{-\frac{1}{\beta}}\log(x)^{-\frac{1}{\beta}}$ for $x\geq 2$, it is $\mu\notin \sP_{\lb}(\bR)$ but $\mu\in\sP_{\frac{1}{\beta}-\varepsilon}$ for every $\varepsilon\in(0,\frac{1}{\beta})$.
\end{proof}

\subsection{Definition of \texorpdfstring{$c$}{}-homogeneous, regular and convex metric} \label{app:metric}

For $a\in\bR$ and $\mu = \cL(X)\in\sPR$ let $a^{\#}\mu := \cL(aX)$. Further, for $\mu,\nu\in\sPR$ let $\mu\ast\nu\in\sPR$ be the convolution.

\begin{definition}\label{def:properties}
    Let $c\in(0,\infty), p\in[1,\infty)$. An extended metric $\dd:\sPR\times\sPR\to[0,\infty]$ is called 
    \begin{itemize}
        \item \underline{$c$-homogeneous} if for all $\gamma\in[0,1]$ and $\mu,\nu\in\sPR$ it is $\dd(\gamma^{\#}\mu,\gamma^{\#}\nu) = \gamma^c\dd(\mu,\nu)$.
        \item \underline{regular} if for all $\mu,\nu,\vartheta\in\sPR$ it is $\dd(\mu\ast\vartheta,\nu\ast\vartheta) \leq \dd(\mu,\nu)$.
        \item \underline{$p$-convex} if for all $m\in\bN$ and $(\alpha_1,\dots,\alpha_m)\in[0,1]^m$ with $\sum_{i=1}^m\alpha_i=1$ and $(\mu_1,\dots,\mu_m)$, ${(\nu_1,\dots,\nu_m)\in\sPR^m}$ it holds that $\dd^p\left(\sum\nolimits_{i=1}^m\alpha_i\mu_i,\sum\nolimits_{i=1}^m\alpha_i\nu_i\right) \leq \sum\nolimits_{i=1}^m\alpha_i\dd^p(\mu_i,\nu_i)$.
        \item \underline{convex} if there exists $p\in[1,\infty)$ such that $\dd$ is $p$-convex.
    \end{itemize}
\end{definition}
In classical literature on probability metrics such as \cite{rachev:91} the properties $c$-homogeneous plus regular are called $(c,+)$-ideal.

\subsection{Proof of Theorem~\ref{thm:bounds-new}} \label{app:bounds}

For $\xi=(x_1,\dots,x_m,y_1,\dots,y_{m-1})\in\Xi_m$ define the map 
$$\pr(\;\cdot\;|\xi):\bR\to\{x_1,\dots,x_m\},\;\;\pr(x|\xi) = x_i :\Leftrightarrow x\in(y_{i-1},y_i].$$
Further, the \emph{push-forward} of a measurable map $f:\bR\to\bR$ is denoted by $f^{\#}:\sPR\to\sPR$ defined as $f^{\#}(\mu) = \mu\circ f^{-1}$. That is, $\mu=\cL(X)\Rightarrow f^{\#}(\mu)=\cL(f(X))$.

\begin{lemma}\label{lemma:basic}
    For every $\mu\in\sPR$ and $\xi = (x_1,\dots,x_m,y_1,\dots,y_{m-1})\in\TNm$ it holds that 
    \begin{itemize}
        \item[(a)] $\PN(\mu,\xi) = \pr^{\#}(\mu|\xi)$ 
        \item[(b)] $F_{\PN(\mu,\xi)} = \sum_{i=1}^{m+1}F_{\mu}(y_{i-1})1_{[x_{i-1},x_i)}(\;\cdot\;)$,
        \item[(c)] $F^{-1}_{\PN(\mu,\xi)} = \pr(F_{\mu}^{-1}(\;\cdot\;)|\xi)$,
    \end{itemize}
    where we set $y_0=x_0=-\infty, y_m = x_{m+1} = \infty, F_{\mu}(-\infty)=0, F_{\mu}(\infty)=1$.
\end{lemma}
\begin{proof}
    (a) is obvious from the definitions and (b) follows from (a) applied to sets of the form $(-\infty,x], x\in\bR$. (c) follows from the fact the inverse quantile function $F^{-1}_{\mu}$ is the unique left-continuous non-decreasing function $(0,1)\to\bR$ such that $\cL(F^{-1}_{\mu}(U))=\mu$ holds for $U$ continuous uniform on $(0,1)$. Now, $\pr(F_{\mu}^{-1}(U)|\xi)$ has distribution $\pr^{\#}(\mu|\xi)$ and $\pr(F_{\mu}^{-1}(\cdot)|\xi)$ is left-continuous and non-decreasing, because both $\pr(\cdot|\xi)$ and $F_{\mu}^{-1}$ are. 
\end{proof}

\noindent
For $\mu = \cL(X)\in\sPR, z\in\bR, w, \beta\in(0,\infty)$ let ${\tailwb(\mu,z,w)\!=\!\int_0^{\infty}x^{\beta-1}\bP[|X\!-\!z|\!>\!w+x]\md x}$ and $\taillb(\mu,z,w)\!=\!\int_0^{\infty}\bP[|X\!-\!z|\!>\!w+x]^{\beta}\md x$.
In particular, $\tailwb(\mu,\xi) = \tailwb(\mu,z(\xi),w(\xi))$, similar for $\taillb$.

\begin{proof}[Proof of Lemma~\ref{lemma:basic-bound}]
    (a) It is $\wb(\PN(\mu,\xi),\mu) = \bE[|\pr(X|\xi)-X|^{\beta}]^{\frac{1}{1\vee\beta}}$ by Lemma~\ref{lemma:basic}. Further, $x\in[x_1,x_m]$ implies $|\pr(x|\xi)-x|\leq 2\delta(\xi)$. Writing 
    $1\equiv 1(X\in[x_1,x_m]) + 1(X<x_1) + 1(X>x_m)$, using linearity of expectation and the formula $\bE[Z^{\beta}] = \beta\int_0^{\infty}x^{\beta-1}\bP[Z>x]\md x$ we have
    \begin{align*}
        \wb\big(\PN(\mu,\xi),\mu\big)^{1\vee\beta} &= \bE[|\pr(X|\xi)-X|^{\beta}]\\
        &\leq 2^{\beta}\delta(\xi)^{\beta} + \bE[(x_1-X)^{\beta}1(X<x_1)] + \bE[(X-x_m)^{\beta}1(X>x_m)]\\
        &=2^{\beta}\delta(\xi)^{\beta} + \beta\int_0^{\infty}x^{\beta-1}\left(\bP[X<x_1-x] + \bP[X>x_m+x]\right)\md x\\
        &=2^{\beta}\delta(\xi)^{\beta} + \beta\int_0^{\infty}x^{\beta-1}\bP\left[\left|X-z(\xi)\right|>w(\xi)+x\right]\md x\\
        &=2^{\beta}\left[\delta(\xi)^\beta + \frac{\beta}{2^{\beta}}\tailwb(\mu,\xi)\right] \leq 2^{\beta + 1} \max\{\delta(\xi)^{\beta}, \frac{\beta}{2^{\beta}}\tailwb(\mu,\xi)\}.
    \end{align*}
    Taking both sides to the power of $1/(1\vee\beta)$ and using $2^{(\beta+1)/(1\vee\beta)}\leq 4$ and $(\frac{\beta}{2^{\beta}})^{1/(1\vee\beta)} \leq 1$ yields the claim.
    
    (b) By Lemma~\ref{lemma:basic} the cdf of $\Pi(\mu,\xi)$ satisfies $x\in[x_{i-1},x_{i})\Rightarrow F_{\Pi(\mu,\xi)}(x) = F_{\mu}(y_{i-1})$ for every $i=1,\dots,m+1$
    where $x_0=-\infty, x_{m+1}=\infty, y_0=-\infty, y_m=\infty$. Hence  
    \begin{align*}
        \lb\big(\PN(&\mu,\xi),\mu\big)^{\beta} = \int_{-\infty}^{\infty}|F_{\mu}(x)-F_{\PN(\mu,\xi)}(x)|^{\beta}\md x\\
        &= \sum_{i=1}^{m+1}\int_{x_{i-1}}^{x_i}|F_{\mu}(x)-F_{\mu}(y_{i-1})|^{\beta}\md x\\
        &\leq \int_{-\infty}^{x_1}F_{\mu}(x)^{\beta}\md x + \int_{x_m}^{\infty}(1-F_{\mu}(x))^{\beta}\md x + \sum_{i=2}^m(x_i-x_{i-1})|F_{\mu}(x_i)-F_{\mu}(x_{i-1})|^{\beta}\\
        &\leq \int_{-\infty}^{x_1}F_{\mu}(x)^{\beta}\md x + \int_{x_m}^{\infty}(1-F_{\mu}(x))^{\beta}\md x + \delta(\xi)\cdot\max_{2\leq i\leq m}|F_{\mu}(x_i)-F_{\mu}(x_{i-1})|^{\beta-1}\\
        &\leq 2\int_{0}^{\infty}\bP\left[\left|X-z(\xi)\right|>w(\xi)+x\right]^{\beta}\md x + \delta(\xi) \leq 4 \max\{\delta(\xi),\taillb(\mu,\xi)\}.
    \end{align*}
    Taking both sides to the power of $(1/\beta)$ yields the claim.
\end{proof}

Let $\eta^{(k)},k\in\bN$ be calculated using (PPA) with hyper-parameter (functions) $M, W, z$ and let $\xi^{(k)} = \xilin(M(k),W(k),z)$. Assume $\eta^{(0)}_s = \delta_z$. 

\begin{lemma}\label{lemma:Tetatail}
    Let $\tail = \tailwb, \beta\in(0,\infty)$ or $\tail = \taillb$ with $\beta\in[1,\infty)$. Then for all $s\in\cS$ and $k\in\bN$ it is
    $$\tail(\cT_s\eta^{(k-1)},\xi^{(k)}) \leq \max_{(a,\sn)\in\cA\times\cS}\;\tail\Big(\cL(R_{sa\sn}),(1-\gamma)z,(1-\gamma)W(k)\Big).$$
\end{lemma}
\begin{proof}
    Write $\eta^{(k-1)} = [\cL(G_{k-1,\sn}):\sn\in\cS]$ with $(G_{k-1,\sn})_{\sn}$ independent of $(R_{sa\sn})_{sa\sn}$ and the full mdp $(S^{(t)},A^{(t)},R^{(t)})_{t\in\bN_0}$, thus $\cL(R^{(0)}+\gamma G_{k-1,S^{(1)}}|S^{(0)}=s) = \cT_s\eta^{(k-1)}$. It holds that
    $$\bP[|R^{(0)}+\gamma G_{k-1,S^{(1)}} - z| > W(k) + x|S^{(0)}=s] \leq \max_{a,\sn}\bP[|R_{sa\sn}+\gamma G_{k-1,\sn} - z| > W(k) + x].$$
    Since $\eta^{(k-1)}_s$ is the output of the previous step of the algorithm, $G_{k-1,\sn}$ is concentrated on $[z-W(k-1),z+W(k-1)]\subseteq [z-W(k),z+W(k)]$, note that we assume $W(k-1)\leq W(k)$. Writing $0 = \gamma z - \gamma z$, triangle inequality yields $|R_{sa\sn}+\gamma G_{k-1,\sn} - z|\leq |R_{sa\sn} - (1-\gamma)z| + \gamma W(k)$. Hence for every $x>0$ it is 
    $$\bP\left[\left|(R_{sa\sn}+\gamma G_{k-1,\sn})-z\right|>W(k)+x\right] \leq \bP\left[\left|R_{sa\sn} - (1-\gamma)z\right|>(1-\gamma)W(k)+x\right].$$ 
    Combining these bounds yields the claim. 
\end{proof}

Let $B(\cdot,\cdot)$ be the Beta function and $\Gamma(\cdot)$ the $\Gamma$-function

\begin{lemma}\label{lemma:tail_bound}
    Let $\mu = \cL(X)\in\sPR$ and $z\in\bR, w>0$. 
    \begin{itemize}
        \item[(a)]
        \begin{align*}
            0<\beta<\alpha &&\Longrightarrow &&\tailwb(\mu,z,w) &\leq \frac{B(\beta,\alpha-\beta)\bE[|X-z|^{\alpha}]}{w^{\alpha-\beta}}\\
            \lambda>0, \beta>0 &&\Longrightarrow &&\tailwb(\mu,z,w) &\leq \frac{\Gamma(\beta)\bE[\exp(\lambda|X-z|)]}{\lambda^{\beta}\exp(\lambda w)}\\
            \bP[|X-z|\leq w]=1 &&\Longrightarrow &&\tailwb(\mu,z,w) &= 0.
        \end{align*}
        \item[(b)]  
        \begin{align*}
            1\leq \frac{1}{\beta}<\alpha &&\Longrightarrow &&\taillb(\mu,z,w) &\leq \frac{\bE[|X-z|^{\alpha}]^{\beta}}{(\alpha\beta-1)w^{\alpha\beta-1}}\\
            \lambda>0, \beta>0 &&\Longrightarrow &&\taillb(\mu,z,w) &\leq \frac{\bE[\exp(\lambda|X-z|)]^{\beta}}{\lambda\beta\exp(\lambda \beta w)}\\
            \bP[|X-z|\leq w]=1 &&\Longrightarrow &&\taillb(\mu,z,w) &= 0.
        \end{align*}
    \end{itemize}
\end{lemma}
\begin{proof}
    In both (a) and (b) the third implication is obvious. For the other implications, apply Markov's inequality as $\bP[|X-z|>w+x] \leq \bE[h(|X-z|)]/h(w+x)$ with functions $h(y)=y^{\alpha}$ (first implications in (a)+(b)) and $h(y)=\exp(\lambda y)$ (second implications in both (a) and (b)) and then explicitly calculate the bounding integrals. For the first implication in (a) notice that $\int_0^{\infty}\frac{x^{\beta-1}}{(w+x)^{\alpha}}\md x = B(\beta,\alpha-\beta)w^{-(\alpha-\beta)}$, the other integrals are more elementary.   
\end{proof}

Still, let $\eta^{(k)}, k\in\bN$ be the output of PPA with parameter functions $M, W, z$ and initial approximations $\eta^{(0)}_s = \delta_z$. In the following, let $\alpha,\lambda\in(0,\infty)$ and define
\begin{equation*}
    P(z,\alpha) = \max_{sa\sn}\bE[|R_{sa\sn}-(1-\gamma)z|^{\alpha}]\;\;\text{and}\;\;E(z,\lambda) = \max_{sa\sn}\bE[\exp(\lambda|R_{sa\sn}-(1-\gamma)z|)].
\end{equation*}
That is, (A1.P$\alpha$) holds iff $P(z,\alpha)<\infty$ and (A1.E$\lambda$) holds iff $E(z,\lambda)<\infty$. Further, set 
$$\delta(k) = \frac{2W(k)}{M(k)-1}.$$

\begin{theorem}\label{thm:bounds}
    \begin{itemize}
        \item[(a)] It is $\er(\wb,k) \leq 4\cdot \max\left\{\delta(k)^{\frac{\beta}{1\vee\beta}},\;T(k)\right\}$ with 
        \begin{itemize}
            \item $T(k) \leq \left[B(\beta,\alpha-\beta) P(z,\alpha)\right]^{\frac{1}{1\vee\beta}}\cdot \left[(1-\gamma)W(k)\right]^{-\frac{\alpha-\beta}{1\vee\beta}}$ in case $\beta\in(0,\alpha)$,
            \item $T(k) \leq \left[\Gamma(\beta) \lambda^{-\beta} E(z,\lambda)\right]^{\frac{1}{1\vee\beta}}\cdot \exp\left(-\frac{\lambda(1-\gamma)}{1\vee\beta}W(k)\right)$ in case $\beta\in(0,\infty)$,
            \item $T(k) = 0$ in case $\bP[\frac{R_{sa\sn}}{1-\gamma}\in [z-W(k),z + W(k)]] = 1$ for all $s,a,\sn$.
        \end{itemize}
        \item[(b)] It is $\er(\lb,k) \leq 4\cdot \max\left\{\delta(k)^{\frac{1}{\beta}},\;T(k)\right\}$ with 
        \begin{itemize}
            \item $T(k) \leq \left[(\alpha\beta-1)^{-1/\beta}P(z,\alpha)\right]\cdot \left[(1-\gamma)W(k)\right]^{-(\alpha-\frac{1}{\beta})}$ in case $\beta\geq 1, \beta>1/\alpha$,
            \item $T(k) \leq \left[(\lambda\beta)^{-1/\beta}E(z,\lambda)\right]\cdot \exp\left(-\lambda(1-\gamma)W(k)\right)$ in case $\beta\geq 1$,
            \item $T(k) = 0$ in case $\bP[\frac{R_{sa\sn}}{1-\gamma}\in [z-W(k),z + W(k)]] = 1$ for all $s,a,\sn$.
        \end{itemize}
    \end{itemize}
\end{theorem}

\begin{proof}
  We consider only $\dd=\wb$, the case $\dd=\lb$ can be shown similarly. Applying Lemma~\ref{lemma:basic-bound} leads to 
    \begin{align*}
        \er(\wb,k) &= \max_{s}\wb(\Pi(\cT_s\eta^{(k-1)},\xi^{(k)}),\cT_s\eta^{(k-1)}) \leq 4\cdot \max\left\{\delta(k)^{\frac{\beta}{1\vee\beta}},\;T(k)\right\}
    \end{align*}
    with $T(k) = \max_{s}\tailwb(\cT_s\eta^{(k-1)},\xi^{(k)})^{\frac{1}{1\vee\beta}}$. Applying Lemma~\ref{lemma:Tetatail} leads to 
    $$T(k)^{1\vee\beta} \leq \max_{s,a,\sn}\tailwb\Big(\cL(R_{sa\sn}),(1-\gamma)z,(1-\gamma)W(k)\Big).$$
    The claimed $\wb$-bounds follow from Lemma~\ref{lemma:tail_bound}.
\end{proof}

\subsection{Proof of Lemma~\ref{lemma:range}} \label{app:range}

Let $\mu=\cL(X)\in\sPR$ and $u\in(0,1)$. Some $x\in\bR$ is a $u$-quantile of $\mu$, that is $x$ satisfies $\bP[X<x]\leq u\leq \bP[X\leq x]$, if and only if $x\in [F^{-1}_{\mu}(u),F^{+1}_{\mu}(u)]$, where the function $F^{+1}_{\mu}:(0,1)\to\bR$ is defined by $F^{+1}_{\mu}(u) = \sup\{x\in\bR| \lim\nolimits_{y\uparrow x}F_{\mu}(y)\leq u\}$. 

\begin{lemma}
    Let $\mu, \nu\in\sPR$ with convolution $\mu\ast\nu\in\sPR$ and let $u\in(0,1)$. 
    \begin{align*}
        F_{\mu\ast\nu}^{-1}(u) &\leq F_{\mu}^{-1}(\sqrt{u})+F_{\nu}^{-1}(\sqrt{u})\\
        F_{\mu\ast\nu}^{+1}(u) &\geq F_{\mu}^{+1}(1-\sqrt{1-u})+F_{\nu}^{+1}(1-\sqrt{1-u}).
    \end{align*}
\end{lemma}
\begin{proof}
    Let $X, Y$ be independent with $\cL(X)=\mu, \cL(Y)=\nu$, thus $\mu\ast\nu=\cL(X+Y)$. Let $x=F_{\mu}^{-1}(\sqrt{u})$ and $y=F_{\nu}^{-1}(\sqrt{u})$, thus $F_{\mu}(x) \geq \sqrt{u}$ and $F_{\nu}(y)\geq \sqrt{u}$. It follows
    \begin{align*}
        F_{\mu\ast\nu}(x+y) &= \bP\left[X+Y\leq x+y\right] \geq \bP[X\leq x, Y\leq y] = F_{\mu}(x)F_{\nu}(y) \geq \sqrt{u}\sqrt{u} = u. 
    \end{align*}
    For all $z\in\bR$ the equivalence $F_{\mu\ast\nu}(z)\geq u\Leftrightarrow z\geq F_{\mu\ast\nu}^{-1}(u)$ holds. Applying this to $z=x+y$ gives the first claimed bound. For the second claim, note $F^{-1}_{\cL(-Z)}(u) = -F^{+1}_{\cL(Z)}(1-u)$ holds for any RV $Z$. Applying the first bound to $(\mu',\nu') = (\cL(-X),\cL(-Y))$ leads, for every $v\in(0,1)$, to
    \begin{align*}
        -F^{+1}_{\mu\ast\nu}(1-v) = F^{-1}_{\cL((-X)+(-Y))}(v) \leq -(F^{+1}_{\mu}(1-\sqrt{v}) + F^{+1}_{\nu}(1-\sqrt{v})).
    \end{align*}
    Multiplying by $(-1)$ and substituting $v = 1 - u$ yields the claim.
\end{proof}

\begin{proof}[Proof of Lemma~\ref{lemma:range}]
    Let $\eta = [\cL(G_{\sn}):\sn\in\cS]$ with $(G_{\sn})_{\sn}$ independent from $(R_{sa\sn})_{sa\sn}$. Then $\cT_s\eta(-\infty,\xmin) = \sum_{a,\sn}\pi(a|s)p(\sn|s,a)\bP[R_{sa\sn}+\gamma G_{\sn} <\xmin]$. For every relevant $(a,\sn)$, using the definition of $\xmin$, the fact that $F^{-1}_{\mu}\leq F^{+1}_{\mu}$ and that $F_{\cL(\gamma X)}^{+1} = \gamma F^{+1}_{\cL(X)}$ in combination with the previous Lemma:  
    \begin{align*}
        \bP[R_{sa\sn}+\gamma G_{\sn} <\xmin] &\leq \bP[R_{sa\sn}+\gamma G_{\sn} <F_{sa\sn}^{-1}\!\left(1-\sqrt{1-\varepsilon_u}\right)\!+\!\gamma\!\cdot\! F_{\eta_{\sn}}^{-1}\!\!\left(1-\sqrt{1-\varepsilon_u}\right)]\\
        &\leq \bP[R_{sa\sn}+\gamma G_{\sn} <F_{sa\sn}^{+1}\!\left(1-\sqrt{1-\varepsilon_u}\right)\!+\!\gamma\!\cdot\! F_{\eta_{\sn}}^{+1}\!\!\left(1-\sqrt{1-\varepsilon_u}\right)]\\
        &\leq \bP[R_{sa\sn}+\gamma G_{\sn} < F^{+1}_{\cL(R_{sa\sn}+\gamma G_{\sn})}(\varepsilon_u)] \leq \varepsilon_u.
    \end{align*}
    To show $\cT_s\eta(\xmax,\infty)\leq \varepsilon_u$ proceed similar and note that for each relevant pair $(a,\sn)$ it is 
    \begin{align*}
        \bP[R_{sa\sn}+\gamma G_{\sn} >\xmax] &=1 - \bP[R_{sa\sn}+\gamma G_{\sn} \leq \xmax]\\
        &\leq 1 - \bP[R_{sa\sn}+\gamma G_{\sn} \leq F_{sa\sn}^{-1}\!\left(\sqrt{1-\varepsilon_u}\right)\!+\!\gamma\!\cdot\! F_{\eta_{\sn}}^{-1}\!\!\left(\sqrt{1-\varepsilon_u}\right)]\\
        &\leq 1 - \bP[R_{sa\sn}+\gamma G_{\sn} \leq F^{-1}_{\cL(R_{sa\sn}+\gamma G_{\sn})}(1-\varepsilon_u)] \leq 1 - (1 - \varepsilon_u) = \varepsilon_u.
    \end{align*}
    The bound $\cT_s\eta[\xmin,\xmax]\geq 1 - 2\varepsilon_u$ follows immediately. 
\end{proof}

\subsection{Choice of the hyper-parameter functions in (\ref{eq:choice})}\label{app_hyper_choice}

The reasoning behind recommendation \eqref{eq:choice} is now explained: Let $\eta^{(n)}$ be the $n$-th output when using arbitrary parameter functions $M, \mathcal{E}_u$. By construction, $\eta^{(n)}_s$ is finitely supported with particles $(x_i,p_i), i=1,\dots,M(n)$, where the $x_i$'s are evenly spaced with distance $|x_{i+1}-x_i|$ being equal to
\begin{equation*}
    \delta(n,s) := \delta\left(A(\eta^{(n-1)},s,n)\right) = \frac{\xmax(n,s)-\xmin(n,s)}{M(n)-1},
\end{equation*}
where $\xmin(n,s), \xmax(n,s)$ are given by \eqref{eq:xmin} and \eqref{eq:xmax} with $\eta = \eta^{(n-1)}$ and $\epsilon_u = \mathcal{E}_u(n)$. In order to obtain high quality approximations, $M$ and $\mathcal{E}_u$ should be chosen such that both $\delta(n,s)\to 0$ and $\mathcal{E}_u(n)\to 0$ as $n\to\infty$. The asymptotic behaviour of $\delta(n,s)$ can be bounded as follows:
\begin{theorem}\label{thm:delta-bound}
    Assume (A1.P$\alpha$) with $\alpha>0$ and set $\mathcal{E}_u(0):=1$. Then for every $s\in\cS$ we have $\delta(n) := \max_{s\in\cS}\delta(n,s) = \cO\left(\frac{1}{M(n)}\sum\nolimits_{k=0}^n \gamma^{n-k}\mathcal{E}_u(k)^{-1/\alpha}\right)$ as $n\to\infty$.
\end{theorem}
\begin{proof}
    Let 
    \begin{align*}
        \xmin(n,s) &= \min_{(a,\sn)}\left[F_{sa\sn}^{-1}\left(1-\sqrt{1-\cE_u(n)}\right) + \gamma F_{\eta^{(n-1)}_{\sn}}^{-1}\left(1-\sqrt{1-\cE_u(n)}\right)\right]\\ 
        \xmax(n,s) &= \max_{(a,\sn)}\left[F_{sa\sn}^{-1}\left(\sqrt{1-\cE_u(n)}\right) + \gamma F_{\eta^{(n-1)}_{\sn}}^{-1}\left(\sqrt{1-\cE_u(n)}\right)\right].
    \end{align*}
    With $H(u) = \max_{sa\sn}\left[F_{sa\sn}^{-1}\left(\sqrt{1-u}\right)-F_{sa\sn}^{-1}\left(1-\sqrt{1-u}\right)\right]$ it is 
    \begin{align*}
        \xmax(n,s)&-\xmin(n,s)\\
        &\leq H(\cE_n(u)) + \gamma \max_{\sn} \left[F_{\eta^{(n-1)}_{\sn}}^{-1}\left(\sqrt{1-\cE_u(n)}\right) - F_{\eta^{(n-1)}_{\sn}}^{-1}\left(1-\sqrt{1-\cE_u(n)}\right)\right]\\
        &\leq H(\cE_n(u)) + \gamma \max_{\sn}\left[\xmax(n-1,\sn)-\xmin(n-1,\sn)\right],
    \end{align*}
    where the second inequality holds since $\eta^{(n-1)}_{\sn}$ is supported on $[\xmin(n-1,\sn),\xmax(n-1,\sn)]$. Hence with $W(n) := \max_{s\in\cS}\left[\xmax(n,s)-\xmin(n,s)\right]$ it is $W(n) \leq H(\cE_n(u)) + \gamma W(n-1)$.
    Assuming each $\eta^{(0)}_s$ it supported on $[a,b]$ and iterating this inequality inductively down to $n=0$ yields $W(n) \leq \sum\nolimits_{k=1}^n \gamma^{n-k}H(\cE_u(k))+\gamma^n (b-a)$ and hence
    $$\delta(n,s) \leq \max_{\sn\in\cS}\delta(n,\sn) \leq \frac{W(n)}{M(n)-1} \leq \frac{1}{M(n)-1}\sum_{k=1}^n \gamma^{n-k}H\left(\mathcal{E}_u(k)\right) + \gamma^n \frac{b-a}{M(n)-1}.$$
    Suppose (A1.P$\alpha$). By Lemma~\ref{lemma:quantile-bound} there exists $C>0$ with $F_{sa\sn}^{-1}(v)-F_{sa\sn}^{-1}(1-v) \leq \frac{C}{(1-v)^{1/\alpha}}$ for all $v\in(0,1)$ and $s,a,\sn$. Using $1/(1-\sqrt{1-u})\leq 2/u$ for all $u\in(0,1)$, it follows $H(u)\leq C\left(\frac{2}{u}\right)^{1/\alpha}$. Thus, setting $\cE_u(0):=1$, the previous display results in
    \begin{align*}
        \delta(n,s) &\leq \frac{1}{M(n)-1}\sum_{k=1}^n \gamma^{n-k}C\left(\frac{2}{\mathcal{E}_u(k)}\right)^{1/\alpha} + \gamma^n \frac{b-a}{M(n)-1},
    \end{align*}
    which is $\cO\left(\frac{1}{M(n)}\sum\nolimits_{k=0}^n \gamma^{n-k}\mathcal{E}_u(k)^{-1/\alpha}\right)$ as claimed.
\end{proof}

Now suppose (A1.P$\alpha$) holds, let $\beta<\alpha$ and set $h = \frac{1\vee\beta}{\beta}\frac{\alpha}{\alpha-\beta} > 1$. Theorem~\ref{thm:delta-bound} shows: Choosing $M(n) = \cO((1/\theta)^{hn})$ and $\mathcal{E}_u(n) = \cO(\theta^{\beta h n})$ leads to $\delta(n) = \cO\left(\theta^{\frac{1\vee\beta}{\beta}n}\right)$, similar to Theorem~\ref{thm:plain-analysis}. For $\varepsilon>0$ let $n(\varepsilon)=\min\{n\in\bN| \max\{\mathcal{E}_u(n),\delta(n)\}\leq\varepsilon\}\in\bN$. Using the obtained asymptotic formulas shows that $n(\varepsilon) \leq \max\{\frac{1}{\beta\cdot h},\frac{\beta}{1\vee\beta}\}\frac{\log(1/\varepsilon)}{\log(1/\theta)} + \cO(1)$ as $\varepsilon\to 0$. Thus, the time to calculate the $n(\varepsilon)$-th approximation is of order $\cO\left((1/\varepsilon)^{2r(\beta,\alpha)}\right)$ asymptotically as $\varepsilon\to 0$, where $r(\beta,\alpha) = \max\{\frac{1}{\beta},\frac{\alpha}{\alpha-\beta}\}$. The function $(0,\alpha)\owns \beta\mapsto r(\beta,\alpha)$ takes its minimum at $\beta = \frac{\alpha}{\alpha + 1}$ with minimal value $r(\frac{\alpha}{\alpha+1},\alpha) = \frac{\alpha + 1}{\alpha}$. Choosing $\beta = \frac{\alpha}{\alpha+1}$ leads to $h = \left(\frac{\alpha+1}{\alpha}\right)^2$ and $\beta h = \frac{1}{\beta} = \frac{\alpha+1}{\alpha}$. Thus, it is reasonable to choose $M(k) = \cO\left((1/\theta)^{\left(\frac{\alpha+1}{\alpha}\right)^2\cdot n}\right)$ and $\mathcal{E}_u(k) = \cO\left(\theta^{\frac{\alpha+1}{\alpha}\cdot n}\right)$ with a time to calculate the $n(\varepsilon)$'th approximation of order $\cO\left((1/\varepsilon)^{2\cdot \frac{\alpha+1}{\alpha}}\right)$ as $\varepsilon\to 0$. For practical purposes, it seems reasonable to suggest this method only in case $\frac{\alpha+1}{\alpha}>1$ is small, where the lower bound $1$ is attained in the limit $\alpha\to\infty$. In the limiting case, it is $M(k) = \cO\left((1/\theta)^k\right)$ and $\mathcal{E}_u(k) = \cO\left(1/M(k)\right)$, which leads to \eqref{eq:choice}.

It remains to show:
\begin{lemma}\label{lemma:quantile-bound}
    Let $\mu=\cL(X)\in\sPaR$. Then for all $u\in(0,1)$ we have 
    $$-\bE[|X|^{\alpha}]^{1/\alpha}\cdot \frac{1}{u^{1/\alpha}}\;\;\leq\;\;F_{\mu}^{-1}(u)\;\;\leq\;\;\bE[|X|^{\alpha}]^{1/\alpha}\cdot \frac{1}{(1-u)^{1/\alpha}}.$$
\end{lemma}
\begin{proof}
    Let $D:=\bE[|X|^{\alpha}]<\infty$. In case $D=0$ it is $\bP[X=0]=1$ and $F_{\mu}^{-1}(u)=0$. Suppose $D>0$. Markov's inequality gives $1-F_{\mu}(x) = \bP[X>x] \leq \bP[|X|>x] \leq Dx^{-\alpha}$ for all $x>0$, hence $F_{\mu}(x)\geq 1 - Dx^{-\alpha}$ for all $x>0$. Applying this to $x = \left(\frac{1-u}{D}\right)^{-\frac{1}{\alpha}}>0$ gives $F_{\mu}(x)\geq u$, hence $x\geq F^{-1}_{\mu}(u)$, which is the second bound. Applying the second bound to $\tilde\mu = \cL(-X)$, $u=1-q\in(0,1)$ and using $F_{\cL(-X)}^{-1}(u) = -F_{\cL(X)}^{+1}(1-u)$ gives $F^{+1}_{\mu}(q)\geq -\bE[|X|^{\alpha}]^{1/\alpha}\cdot \frac{1}{q^{1/\alpha}}$. The lower bound follows from this by $F_{\mu}^{-1}(q)\geq \limsup_{q'\uparrow q}F^{+1}(q')\geq \limsup_{q'\uparrow q}-\bE[|X|^{\alpha}]^{1/\alpha}\cdot \frac{1}{(q')^{1/\alpha}} = -\bE[|X|^{\alpha}]^{1/\alpha}\cdot \frac{1}{q^{1/\alpha}}$.
\end{proof}

\subsection{Proofs of Theorem~\ref{thm:density} and Corollary \ref{cor:density}}\label{app_proof_24}

\begin{proof}[Theorem~\ref{thm:density}]
    The stochastic process $\bSAS = (S^{(t)},A^{(t)},S^{(t+1)})_{t\in\bN_0}$ is a random variable taking values in the infinite product space $(\cS\times\cA\times\cS)^{\bN_0}$, which has elements of the form $\bsas=(s_t,a_t,\sn_t)_{t\in\bN_0}$. Let $\Upsilon := \bP_{s_0}[\bSAS\;\in\cdot\;]$ be its distribution. Recall that $T=\min\{t\in\bN_0|(S^{(t)},A^{(t)},R^{(t)})\in\cG\}\in\bN_0\cup\{\infty\}$. Letting $\cH\subseteq(\cS\times\cA\times\cS)^{\bN_0}$ be the subset of sequences visiting $\cG$ we have  $\{T<\infty\} = \{\bSAS\in\cH\}$ and hence $\Upsilon(\cH)=1$ by assumption. Recall that $G^* = \gamma^TR^{(T)} + Z$ with $Z=\sum_{t\in\bN_0\setminus\{T\}}\gamma^t R^{(t)}$ and $(S,A,\bar S) := (S^{(T)},A^{(T)},S^{(T+1)})\in\cG$. Let $\bP_{\bsas}[(T,S,A,\bar S,R^{(T)},Z)\in\cdot], \bsas\in (\cS\times\cA\times\cS)^{\bN_0}$ be a regular conditional distribution of $(T,S,A,\bar S,R^{(T)},Z)$ given $\bSAS$. In particular, for every measurable set $B\subseteq\bR$ we have
    \begin{equation}\label{eq:density-first}
        \eta^*_{s_0}(B) = \bP_{s_0}[G^*\in B] = \int_{\cH}\bP_{\bsas}[\gamma^TR^{(T)} + Z \in B]\md \Upsilon(\bsas). 
    \end{equation}
    For $\Upsilon$-almost all sequences $\bsas\in \cH$ we have, under $\bP_{\bsas}$, that the random variables $T\in\bN, (S,A,\bar S)\in\cG$ are deterministic, $R^{(T)}\sim r(\cdot|S,A,\bar S)$ has PDF $f_{(S,A,\bar S)}$ and $R^{(T)}, Z$ are independent. Recall that, if $a>0, b\in\bR$ are constants and $X$ has PDF $f$, then $aX+b$ has PDF $x\mapsto \frac{1}{a}f(\frac{x-b}{a})$. If further $Y$ is a RV independent of $X$, then $aX+Y$ has PDF $x\mapsto \bE\left[\frac{1}{a}f(\frac{x-Y}{a})\right]$, which follows from Fubini's theorem. Hence, for $\Upsilon$-almost all sequences $\bsas\in \cH$ it holds that the distribution of $G^* = \gamma^{T}R^{(T)} + Z$ under $\bP_{\bsas}$ has density $x\mapsto \bE_{\bsas}[\frac{1}{\gamma^T}f_{(S,A,\bar S)}\left(\frac{x - Z}{\gamma^T}\right)]$, where $T, (S,A,\bar S)$ are non-random under $\bP_{\bsas}$ with values depending on $\bsas$. Hence, for every measurable $B\subseteq\bR$ it is
    \begin{align*}
        \int_{\cH}\bP_{\bsas}[\gamma^TR^{(T)} + Z \in B]\;\md \Upsilon(\bsas) &= \int_{\cH}\int_B \bE_{\bsas}\left[\frac{1}{\gamma^T}f_{(S,A,\bar S)}\left(\frac{x - Z}{\gamma^T}\right)\right]\;\md x\;\md \Upsilon(\bsas)\\
        &= \int_B \int_{\cH}\bE_{\bsas}\left[\frac{1}{\gamma^T}f_{(S,A,\bar S)}\left(\frac{x - Z}{\gamma^T}\right)\right]\;\md \Upsilon(\bsas)\;\md x \\
        &= \int_B \bE_{s_0}\left[\frac{1}{\gamma^T}f_{(S,A,\bar S)}\left(\frac{x - Z}{\gamma^T}\right)\right]\;\md x,
    \end{align*}
    where the second equation holds by Fubini's theorem and the third is the key property of regular conditional distributions. The claim follows in combination with \eqref{eq:density-first}.
\end{proof}

\begin{proof}[Corollary~\ref{cor:density}]
    (a) Let $M>0$ be such that $|f_{(s,a,\sn)}(x)|\leq M$ for all $(s,a,\sn)\in\cG, x\in\bR$. Since $(S,A,\bar S)\in\cG$, it is $|f_{(S,A,\bar S)}(x)|\leq M$ for all $x\in\bR$ and thus, for every $x\in\bR$, it follows $|f^*_{s_0}(x)| \leq M\cdot \bE_{s_0}[\gamma^{-\tauG}] = M\cdot \Psi(\log(1/\gamma))<\infty$. For (b), let $C>0$ be such that $|f_{(s,a,\sn)}(x)-f_{(s,a,\sn)}(y)|\leq C |x-y|^{\tau}$ for all $x,y\in\bR$ and $(s,a,\sn)\in \cG$. This bound is also satisfied for the random $(S,A,\bar S)\in\cG$ and thus, for all $x, y\in\bR$,
    \begin{align*}
        |f^*_{s_0}(x)-f^*_{s_0}(y)|&\leq  \bE_{s_0}\Big[\frac{1}{\gamma^{\tauG}}\Big|f_{(S,A,\bar S)}\left(\frac{x-Z}{\gamma^{\tauG}}\right) - f_{(S,A,\bar S)}\left(\frac{y-Z}{\gamma^{\tauG}}\right)\Big|\Big]\\
        &\leq  \bE_{s_0}\Big[\frac{1}{\gamma^{\tauG}}\cdot C\cdot \Big|\frac{x-Z}{\gamma^{\tauG}} - \frac{y-Z}{\gamma^{\tauG}}\Big|^{\tau}\Big]\\
        &= C\cdot\bE_{s_0}\big[\gamma^{-(\tau+1)\tauG}\big]\cdot |x-y|^{\tau} = C\cdot \Psi\big((\tau+1)\log(\gamma^{-1})\big)\cdot\;|x-y|^{\tau},
    \end{align*}
    with $C\cdot \Psi\big((\tau+1)\log(\gamma^{-1})\big)<\infty$ by assumption.
\end{proof}

\vskip 0.2in
\bibliography{literature}

\end{document}